\documentclass[11pt]{article}

\usepackage{amsthm,amsfonts,amsmath,amssymb,epsfig,color,float,graphicx,verbatim, enumitem}
\usepackage{multirow}

\usepackage{ifthen}
\usepackage{mdwlist}
\usepackage{amsmath,amssymb,amsfonts,amsthm}
\usepackage{bm}
\usepackage{mathtools}
\usepackage{url}
\usepackage{graphicx}
\usepackage{xspace}
\usepackage{verbatim}

\usepackage{color}
\usepackage{latexsym}
\usepackage{epsfig}
\usepackage{bm}
\usepackage[noend]{algpseudocode}
\usepackage{algorithm}

\def\eps{\epsilon}

\def\to{\rightarrow}

\def\cale{{\cal E}}

\newcommand{\erfc}{\mbox{\text erfc}}

\newcommand{\prob}[2][]{\text{\bf Pr}\ifthenelse{\not\equal{}{#1}}{_{#1}}{}\!\left[#2\right]}
\newcommand{\expect}[2][]{\text{\bf E}\ifthenelse{\not\equal{}{#1}}{_{#1}}{}\!\left[#2\right]}

\newtheorem{theorem}{Theorem}[section]
\newtheorem{lemma}[theorem]{Lemma}

\newtheorem{proposition}[theorem]{Proposition}
\newtheorem{corollary}[theorem]{Corollary}
\newtheorem{claim}[theorem]{Claim}
\newtheorem{definition}[theorem]{Definition}
\newtheorem{fact}[theorem]{Fact}

\newcommand{\ignore}[1]{}

\newcommand{\bg}[1]{\medskip\noindent{\bf #1}}
\definecolor{Red}{rgb}{1,0,0}

\newcommand{\oldbound}[1]{{}}

\renewcommand{\epsilon}{\varepsilon}

\DeclareMathOperator{\R}{\mathbb{R}}

\DeclareMathOperator{\Z}{\mathbb{Z}}

\DeclareMathOperator*{\E}{\mathbb{E}}

\DeclareMathOperator{\poly}{poly}
\DeclareMathOperator{\polylog}{polylog}

\DeclareMathOperator{\normal}{\mathcal{N}}

\renewcommand{\[ }{\begin{eqnarray*}}
\renewcommand{\]}{\end{eqnarray*}}

\definecolor{darkpastelred}{rgb}{0.76, 0.23, 0.13}

\algnewcommand\INPUT{\item[{\textbf {input:}}]}
\algnewcommand\OUTPUT{\item[{\textbf{output:}}]}
\def\shownotes{1}  %set 1 to show author notes
\ifnum\shownotes=1
\newcommand{\authnote}[2]{$\ll$\textsf{ #1 notes: #2}$\gg$}
\else
\newcommand{\authnote}[2]{}
\fi

\newcommand{\jnote}[1]{{\color{red}\authnote{Jason}{#1}}}

\def\colorful{0}

\usepackage{amsfonts,amsmath,amsthm,amssymb,color,float,graphicx,verbatim, enumitem}
\usepackage{multirow}
\newcommand{\eqdef}{\stackrel{{\mathrm {\footnotesize def}}}{=}}

\ifnum\colorful=1
\newcommand{\new}[1]{{\color{red} #1}}

\else
\newcommand{\new}[1]{{#1}}

\fi

\title{Efficient Algorithms and Lower Bounds \\
for Robust Linear Regression}

\author{
Ilias Diakonikolas\thanks{Supported by NSF Award CCF-1652862 (CAREER) and a Sloan Research Fellowship.}\\
University of Southern California\\
{\tt diakonik@usc.edu}\\
\and
Weihao Kong\thanks{Weihao's contributions were supported in part by NSF Award CCF-1704417 and by ONR Award N00014-17-1-2562. Some of this work was done while visiting USC.}\\
Stanford University\\
{\tt whkong@stanford.edu}\\
\and
Alistair Stewart\\ University of Southern California\\
{\tt stewart.al@gmail.com}
}

\begin{document}
\maketitle

\setcounter{page}{0}

\thispagestyle{empty}

\begin{abstract}
We study the prototypical problem of high-dimensional linear regression in a robust model where an $\eps$-fraction
of the samples can be adversarially corrupted. We focus on the fundamental setting where the covariates
of the uncorrupted samples are drawn from a Gaussian distribution $\mathcal{N}(0, \Sigma)$ on $\R^d$. 
We give nearly tight upper bounds and computational lower bounds for this problem.
Specifically, our main contributions are as follows:
\begin{itemize}
\item For the case that the covariance matrix is known to be the identity, 
we give a sample near-optimal and computationally efficient algorithm that draws $\tilde{O}(d/\eps^2)$ 
labeled examples and outputs a candidate hypothesis vector $\widehat{\beta}$ that approximates the unknown regression vector 
$\beta$ within $\ell_2$-norm $O(\eps \log(1/\eps) \sigma)$, where $\sigma$ is the standard deviation
of the random observation noise. An error of $\Omega (\eps \sigma)$ is information-theoretically
necessary, even with infinite sample size. Hence, the error guarantee of our algorithm is
optimal, up to a logarithmic factor in $1/\eps$. 
Prior work gave an algorithm for this problem with sample complexity $\tilde{\Omega}(d^2/\eps^2)$ 
whose error guarantee scales with the $\ell_2$-norm of $\beta$.

\item For the case of unknown covariance $\Sigma$,
we show that we can efficiently achieve the same error guarantee of $O(\eps \log(1/\eps) \sigma)$, 
as in the known covariance case, using an additional $\tilde{O}(d^2/\eps^2)$ unlabeled examples. 
On the other hand, an error of $O(\eps \sigma)$ can be information-theoretically attained with $O(d/\eps^2)$ samples.
We prove a Statistical Query (SQ) lower bound providing evidence that this quadratic 
tradeoff in the sample size is inherent. More specifically, we show that any polynomial time 
SQ learning algorithm for robust linear regression (in Huber's contamination model)
with estimation complexity $O(d^{2-c})$,  where $c>0$ is an arbitrarily small constant, 
must incur an error of $\Omega(\sqrt{\eps} \sigma)$.
\end{itemize}

%Our robust learning algorithm builds on the filtering technique of 
\end{abstract}

\newpage

\section{Introduction}

\subsection{Background and Problem Definition}
Linear regression is a prototypical problem in statistics with a range of applications
in signal processing (e.g., face recognition, time series analysis) and various other data analysis tasks.
The reader is referred to~\cite{Rousseeuw:1987, BhatiaJK15} and references therein. In the realizable case,
linear regression is well-understood. Here we study the problem in a robust model 
where an $\eps$-fraction of the samples are adversarially corrupted. We explore the tradeoff
between sample complexity, computational complexity, and robustness in high-dimensional linear regression,
obtaining both efficient algorithms and nearly matching computational-statistical-robustness tradeoffs.

Estimation in the presence of outliers is an important goal in 
statistics and has been systematically studied within the robust statistics community 
since \cite{Huber64}. Nevertheless, until recently, all known efficient 
estimators could only tolerate a negligible fraction of outliers in high-dimensional settings, 
even for the simplest statistical tasks. Recent work in the theoretical
computer science community~\cite{KLS09,ABL14, DKKLMS16, LaiRV16} gave the first efficient robust
estimators for basic high-dimensional statistical tasks, including learning linear separators, 
and mean and covariance estimation. Since the dissemination of~\cite{DKKLMS16, LaiRV16}, 
there has been a flurry of research activity on robust learning (see Section~\ref{sec:related-work} for a discussion).

In the remaining of this section, we describe our formal setup.
In the realizable setting, the problem of linear regression is defined as follows: We observe a multiset of 
labeled samples $(X_i, y_i)$, where $X_i \in \R^d$ and $y_i \in \R$.
It is assumed that there exists an unknown distribution $D \in \mathcal{D}$, where $\mathcal{D}$ is a known family 
of distributions over $\R^d$, such that $X_i \sim D$. Moreover, there exists an unknown vector $\beta \in \R^d$ 
such that $$y_i = \beta \cdot X_i + \eta_i \;,$$
where $\eta_i$ is some kind of random observation noise. The goal is to compute a hypothesis vector $\widehat{\beta}$
such that $\|\widehat{\beta} - \beta\|_2$ is small.
In this work, we study the fundamental setting that $\mathcal{D}$ is $\normal(0, \Sigma)$, 
where the covariance matrix $\Sigma$ is either a priori known or unknown to the algorithm. 
For simplicity, we also assume that $\eta_i \sim \normal(0,\sigma^2)$ and is independent of $X_i$. 

We consider the following model of robust estimation
that generalizes other existing models, including Huber's contamination model:
\begin{definition} \label{def:adv}
Given $\eps > 0$ and a family of probabilistic models $\mathcal{M}$, 
the \emph{adversary} operates as follows: The algorithm specifies some number of samples $m$.
The adversary generates $m$ samples $X_1, X_2, \ldots, X_m$ from some (unknown) $M \in \mathcal{M}$.
The adversary is allowed to inspect the samples, removes $\eps m$ of them, 
and replaces them with arbitrary points. This set of $m$ points is then given to the algorithm.
We say that a set of samples is $\eps$-corrupted 
if it is generated by the aforementioned process.
\end{definition}

In summary, the adversary is allowed to inspect the samples before corrupting them,
both by adding corrupted points and deleting uncorrupted points. In contrast, in Huber's model
the adversary is oblivious to the samples and is only allowed to add corrupted points.
%We remark that there are no computational restrictions on the adversary.
%The goal is to return the parameters of a model $\widehat{M}$ in $\mathcal{M}$
%that are close to the true parameters in an appropriate metric.

In the context of robust linear regression studied in this paper, 
the adversary can change an arbitrary $\eps$-fraction of the labeled
samples $(X_i, y_i)$ that satisfy the aforementioned definition of linear regression.
The goal is to output a hypothesis vector $\widehat{\beta}$ such that 
$\|\widehat{\beta} - \beta\|_2$ is as small as possible.

\subsection{Our Results and Techniques} \label{ssec:results}

\subsubsection{Robust Learning Algorithms}

We state our positive results for the case of known and identity covariance matrix, $\Sigma = I$.
Our first positive result is a robust learning algorithm for linear regression
that has near-optimal sample complexity, runs in polynomial time, and achieves
an error guarantee that scales with the $\ell_2$-norm of the target regression vector:

\begin{theorem}[Basic Algorithm for Robust Linear Regression] \label{thm:alg-lr-gaussian}
Let $S'$ be an $\eps$-corrupted set of labeled samples of size $\Omega((d/\eps^2) \polylog(\frac{d}{\eps \tau}))$.
There exists an efficient algorithm that on input $S'$ and $\eps>0$, returns a candidate vector $\widehat{\beta}$
such that with probability at least $1-\tau$ it holds $\|\widehat{\beta} - \beta\|_2  = O(\sigma_y \eps \log(1/\eps))$,
where $\sigma_y=\sqrt{\sigma^2 +\|\beta\|_2^2}$).
\end{theorem}

Roughly speaking, the algorithm establishing Theorem~\ref{thm:alg-lr-gaussian} relies on the observation
that robust linear regression can be reduced to robust mean estimation. The main drawback
of this approach is that the error guarantee depends on $\|\beta\|_2$ and in particular does not go to $0$
when $\sigma$ goes to $0$. To eliminate $\|\beta\|_2$ from the RHS while retaining a near-optimal
sample complexity, we require a more sophisticated
approach. Specifically, our main algorithmic contribution is as follows:

\begin{theorem}[Main Algorithm for Robust Linear Regression] \label{thm:alg-lr-gaussian-2}
Let $S'$ be an $\eps$-corrupted set of labeled samples of size $\Omega((d/\eps^2) \polylog(\frac{d}{\eps \tau}))$.
There exists an efficient algorithm that on input $S'$ and $\eps>0$, returns a candidate vector $\widehat{\beta}$
such that with probability at least $1-\tau$ it holds $\|\widehat{\beta} - \beta\|_2  = O(\sigma \eps \log(1/\eps))$.
\end{theorem}

We note that an error of $\Omega(\sigma \eps)$ is information-theoretically necessary for this problem, 
even when the sample size is unbounded (see, e.g.,~\cite{Gao17-reg}). Hence, our second algorithm
achieves the minimax optimal error, up to a logarithmic factor in $1/\eps$.  

We note that the unknown covariance case can be easily reduced to the known covariance case as follows:
First, we robustly learn the covariance matrix in the appropriate metric using $\tilde{O}(d^2/\eps^2)$ samples.
Then, we observe that our above algorithms also work when the covariance matrix $\Sigma$ is only 
approximately known (see Appendix~\ref{sec:approxcov}).  
As we will explain later in this paper, it follows from our computational lower bounds (Theorem~\ref{thm:SQ-lb}) 
that this simple approach to handle the unknown covariance case cannot be improved 
for polynomial-time Statistical Query (SQ) algorithms.

\paragraph{Intuition behind the Algorithm.}
We now provide some intuition of our algorithm. We start with the first algorithm which works 
in the simplest setting where $X\sim \mathcal{N}(0,I)$, $y \sim \beta^TX+\eta$, and $\eta$ has mean $0$ variance $\sigma^2$, 
and achieves estimation accuracy $\|\widehat{\beta}-\beta\|_2 \le \tilde{O}(\eps \sqrt{\|\beta\|+\sigma^2})$. 
%This basic algorithm builds on the filtering framework of~\cite{DKKLMS16}.
The basic algorithm relies on the fact that $yX$ is an unbiased estimator of $\beta$ (indeed, $\E[yX] = \E[X(X^T\beta+\eta)] = \beta$). 
By standard concentration results, given $O(d/\eps^2)$ samples, 
the empirical average of $yX$ yields an estimate of $\beta$ with $\eps(\|\beta\|^2+\sigma^2)$ $\ell_2$-error. 
However, in the presence of adversarial corruptions, the question is now how to robustly estimate the mean 
of the distribution of $yX$. To do this, we apply the filtering technique of~\cite{DKKLMS16}. 
The key observation that enables the filtering framework to apply is that an $\eps$-fraction of corruptions 
cannot corrupt the mean of a distribution by a lot without affecting its covariance. For example, 
in the one-dimensional setting, in order to change the mean by at least a constant, 
the average distance of the corrupted samples to the true mean must be $\Theta(\frac{1}{\eps})$, 
which will yield a $\Theta(\eps(\frac{1}{\eps})^2) = \Theta(\frac{1}{\eps})$ change of the variance. 
This intuition carries over to the high-dimensional setting and forms the basis of our first algorithm.

%where now the fact becomes that the corruption can not change the mean by a lot without also changing the covariance. 

With $O(d/\eps^2)$ samples, the empirical covariance matrix of $yX$ also concentrates around the true covariance matrix. 
Hence, the above key observation provides an indicator to determine whether the mean may be corrupted. 
However, even if we manage to detect the abnormality by looking at the empirical covariance, 
it's still not clear how we are able to fix it. Luckily, this many samples are enough to obtain strong empirical tail bounds. 
More concretely, the empirical tail will be close to the true tail in any direction. 
Combing these two facts, in the case that the empirical covariance is abnormal, we can look 
at the top principal component direction and check if the data satisfies the desired tail bound. 
We claim that since the variance is abnormally large, there must a threshold where the tail bound is significantly violated. 
Hence, we can throw away samples above this threshold which mostly consists of outliers.

This approach gives an $\tilde{O}(\eps \sqrt{\|\beta\|^2+\sigma^2})$ error guarantee 
due to the fact that the operator norm of the covariance of $yX$ is $O(\sqrt{\|\beta\|_2^2+\sigma^2})$. 
However, this error bound is far from the information-theoretic optimal error of $\Theta(\eps \sigma)$ 
when $\|\beta\|_2$ is large. 
A natural idea to circumvent this issue is to boost the accuracy by setting $y' = y-\widehat{\beta}^T X$ 
using the output of the first iteration and again run the basic algorithm. 
Indeed, this approach yields an estimate with near-optimal error $\tilde{O}(\eps\sigma)$ 
by running the algorithm $O(\log(\|\beta\|_2))$ times. However, this scheme seems to require 
$O(\frac{\log(\|\beta\|_2)d}{\eps^2}\poly\log(d/\eps\tau))$ samples, i.e., depending on $\|\beta\|_2$. 

To remove the $\log(\|\beta\|_2)$ dependence in the sample complexity, 
one may consider subtracting $\widehat{\beta}^TX$ from $y$ using 
the same batch of samples repeatedly. The main problem with doing this naively 
is that, for arbitrary $\beta'$, we would need 
the second moment matrix of $(y-\beta'^T X)X $ to be close to its expectation. 
Unfortunately, we are not going to get this guarantee for all $\beta'$ with fewer than $\Omega(d^2)$ samples. 
To see this, let $\|\beta\|_2=O(1)$ and consider that with high probability 
one of our samples $X_1$ will have $\|X_1\|_2=\Theta(\sqrt{d})$ and $X_1 \cdot \beta = O(1)$.  
Now if we let $\beta'=\frac{X_1}{\|X_1\|_2}$ and consider $v = \beta'$, 
we would have $\|(y_1-\beta'^T X_1)^2 X_1 X_1^T\|_2 =\Theta( (v \cdot X_1)^4) = \Theta(d^2)$. 
Since the expected second moment matrix has operator norm $O(1)$, we need $\Omega(d^2)$ 
samples to achieve concentration. Somewhat surprisingly, the only problem 
that prevents the empirical covariance from achieving the desired concentration 
is the samples with large $y-\beta'^TX$, as illustrated in the previous example. 
The concentration property holds if we temporally ignore samples 
with large $|y-\beta' \cdot X|$, thus we can run the previous algorithm 
on the same batch of samples repeatedly. Notice that we will need to add the ignored samples 
back at the end of each iteration, since these samples may contain 
more good samples than corrupted samples. 

Our final modification of the algorithm focuses on removing the $O(\log(\|\beta\|_2))$ dependence in the running time. 
Instead of starting from $\beta'=0$ and running the subroutine algorithm repeatedly, 
we start from the ordinary least squares estimator. Let $U$ be the set of samples that has large $y-{\beta'}^TX$. 
We pre-process by running filter algorithms similar to one used in \cite{DKKLMS16, DKK+17} on the Gaussians $y-\beta'^T X$ and $X$. 
The property of these two filters combined implies that ignoring samples with a large $y-{\beta'} ^T X$ 
does not change the empirical mean of $(y-{\beta'}^T X)X$ by much. Notice that in the case where $\beta'$ 
is the ordinary least square estimator, the empirical mean of $(y-\beta'^TX)X$ is $0$, 
Hence once the the samples passed the pre-process and the filtering step we described in the last paragraph, 
the algorithm can terminate and output the estimate $\widehat{\beta}$ with $\ell_2$ error independent of  $\|\beta\|$.

\subsubsection{Statistical Query Lower Bounds} 

In this section, we describe our Statistical Query (SQ) lower bounds establishing 
a tradeoff between sample complexity and computation complexity for robust linear regression
with unknown (bounded) covariance. 

We start with some basic background.
A Statistical Query (SQ) algorithm relies on an oracle that given any bounded function 
on a single domain element provides an estimate of the expectation of the function on a random sample from the input distribution. 
This computational model was introduced by Kearns~\cite{Kearns:98} in the context of supervised learning as a natural restriction of the PAC model~\cite{Valiant:84}. Subsequently, the SQ model has been extensively studied in a plethora of contexts (see, e.g., ~\cite{Feldman16b} and references therein). We remark that all recently developed algorithms for robust high-dimensional estimation
fit in the SQ framework.

A recent line of work \cite{Feldman13, FeldmanGV15, FeldmanPV15, Feldman16} developed a framework of SQ algorithms for search problems over distributions, which encompasses the linear regression problem studied here. It turns out that one can prove unconditional lower bounds on the computational complexity of SQ algorithms via the notion of {\em Statistical Query dimension}. This complexity measure was introduced in~\cite{bfjkmr94} for PAC learning of Boolean functions and was recently generalized to the unsupervised setting~\cite{Feldman13, Feldman16}. A lower bound on the SQ dimension of a learning problem provides an unconditional lower bound on the computational complexity of any SQ algorithm for the problem.

As our main negative result in this paper, 
we prove a Statistical Query (SQ) lower bound giving evidence that if $X$ has an unknown (bounded) covariance, 
it is computationally hard to approximate $\beta$ well given significantly fewer than $d^2$ samples. 
The reason that this result is interesting is because this learning problem can be information-theoretically solved
with $d$ samples to optimal accuracy. More concretely, we prove (see Theorem~\ref{thm:SQ-lb} 
for a more detailed formal statement):

\begin{theorem}[SQ lower bound, informal statement] \label{thm:SQ-lb-informal}
No SQ algorithm for robust linear regression for Gaussian covariates with unknown bounded covariance and 
random noise with $\sigma^2 \leq 1$
can output a candidate $\beta'$ with $\|\beta'-\beta\|_2 \leq o(\sqrt{\eps})$ on all instances
unless it uses $2^{\Omega(d)}$ statistical queries or each query requires $\Omega(d^2)$ samples to be simulated.
\end{theorem}

We note that $O(d/\eps^2)$ samples information-theoretically
suffice to achieve error $O(\eps \sigma)$ (see, e.g.,~\cite{Gao17-reg}) even 
in the unknown covariance setting. Moreover, as explained in Appendix~\ref{sec:approxcov}, 
with $\tilde O(d^2/\eps^2)$ samples we can {\em efficiently} achieve
error $\tilde{O}(\eps)$ for unknown covariance as well. Hence, Theorem~\ref{thm:SQ-lb-informal} establishes
an inherent tradeoff between computational complexity, sample complexity, and error guarantee
for any SQ algorithm for this problem.

To prove this result, we require a generalization of the technique in~\cite{DKS17-sq}, which 
was designed for unsupervised learning problems.
That work established SQ  lower bounds for {\em unsupervised} learning problems 
using a construction consisting of distributions 
which are standard Gaussians in all except one direction, 
by showing that if such a distribution agrees with the first few moments of the standard Gaussian, 
then it is hard to find the hidden direction. 

As already mentioned, \cite{DKS17-sq} considers unsupervised learning problems 
and the distribution in the construction is a suitable model for $X$, 
but not for the joint distribution $(X,y)$, where the direction of the $y$ coordinate is very different. 
We instead try to make $X$ {\em conditioned on $y$} match moments with a Gaussian. 
When the momets match for all $y$, it is hard to learn the direction $\beta$ 
where these conditional distributions are different. 
We show that when $\|\beta\|_2 = O(\sqrt{\eps})$, we can match three moments 
and it is hard to find the direction of $\beta$. Thus, we obtain a lower bound 
that says that we cannot approximate $\beta$ to within $o(\sqrt{\eps})$ 
with fewer than exponential in $d$ statistical queries, unless we use queries of precision greater 
than we could simulate with $O_{\eps}(d^2)$ samples.

%Other random things we want to say:
%\begin{itemize}
%\item Our focus is on linear regression, even though our method may generalize to other problems. Our goal is to pin-down
%the possibilities and limitations of efficient robust estimation for this fundamental problem.
%\item Our algorithms work in the ``nasty'' corruptions model. Even in Huber's contamination model, all our upper bounds are new.
%We work in this stronger model because we can. (This non-iid model is motivated by applications in adversarial ML.)
%On the other hand, our SQ lower bounds all apply in Huber's model.

%\item Need a short and to the point punchline of how we relate to prior work, which should be clearly stated early.
%Pointer to explicit comparison in relevant subsection.

%\item Cite additional regression books and more applications in there.

%\end{itemize}

\subsection{Comparison to Prior Work} \label{sec:related-work}
Since the initial works~\cite{DKKLMS16, LaiRV16},
there have been a considerable number of papers 
on a wide range of topics related to robust high-dimensional estimation, 
including: learning graphical models~\cite{DiakonikolasKS16b}, 
understanding computation-robustness tradeoffs~\cite{DKS17-sq, DKKLMS17}, 
giving applications to exploratory data analysiis~\cite{DKK+17},
establishing connections to supervised PAC learning~\cite{DKS17-nasty}, 
tolerating more noise by outputting a list of candidate 
hypotheses~\cite{CSV17, DKS17-mixtures}, learning sparse models~\cite{BDLS17},
and robust estimation via sum-of-squares ~\cite{KS17, HL17, KStein17}. 

In the context of the estimation task studied in this paper, 
\cite{BhatiaJK15, BhatiaJKK17} have proposed efficient algorithms for ``robust'' linear regression.  
However, these works consider a restrictive corruption model 
that only allows adversarial corruptions to the responses (but not the covariates).
\cite{BDLS17} studies (sparse) linear regression under Huber's contamination model.
The three main differences between the results of \cite{BDLS17} and our work are as follows: 
(1) The error guarantee provided in~\cite{BDLS17} scales with $\|\beta\|_2$, 
the $\ell_2$-norm of the regression vector. This multiplicative dependence on $\|\beta\|_2$ 
is not information-theoretically necessary (see, e.g.,~\cite{Gao17-reg}). In contrast, the error guarantee
of our algorithm has no dependence on $\|\beta\|_2$, matching the information-theoretic bound
$\Omega(\sigma\eps)$ up to a $\log(1/\eps)$ factor. (2) The algorithm employed in~\cite{BDLS17}, 
building on the convex programming method of~\cite{DKKLMS16}, 
makes essential use of the ellipsoid method (whose separation oracle is another convex program), 
hence is not scalable in high dimensions. (3) Moreover, the~\cite{BDLS17} algorithm assumes 
a priori knowledge of an upper bound $\rho$ on $\|\beta\|_2$, and its running time depends polynomially 
in the {\em magnitude} of $\rho$. Our algorithm relies on an iterative spectral approach and does not
need any assumptions on $\|\beta\|_2$.
(4) The sample complexity of the~\cite{BDLS17} algorithm scales 
quadratically in the dimension, while our algorithm has near-linear (and therefore near-optimal) 
sample complexity. 

%We note that~\cite{CSV17} studied stochastic convex optimization problems
%in an adversarial setting where the majority of the data is corrupted and gave a list-decoding algorithm 
%under similar assumptions to ours. A major difference with our work is that the error guarantees 
%of the~\cite{CSV17} algorithm do not converge to $0$ when the fraction $\eps$ of corruptions tends to $0$.
%Another difference is that their algorithm requires to solve an SDP (in a black-box manner), which prevents scalability.
%A related line of work provides robust algorithms for learning linear separators with malicious 
%noise~\cite{KLS09,ABL14}. 

\subsection{Concurrent and Independent Works}
Three recent works~\cite{prasad2018robust, diakonikolas2018sever, klivans2018efficient} 
provide robust efficient algorithms for linear regression in various settings. 
These works make weaker distributional assumptions on the uncorrupted data
and as a result provide weaker error guarantees, in most cases scaling with $\sqrt{\eps}$, 
as opposed to $\tilde{O}(\eps)$ in our setting. The algorithms in~\cite{prasad2018robust, diakonikolas2018sever} 
succeed in the oversampled regime in the sense that their sample complexities
are at least quadratic in the dimension. The algorithm in~\cite{klivans2018efficient} relies on the SOS
convex programming hierarchy.

\subsection{Structure of this Paper}

In Section~\ref{sec:main-algorithm}, we describe our robust algorithms and in Section~\ref{sec:main-sq}
we give our SQ lower bounds. For the clarity of the presentation, most proofs 
are deferred to the appendix.

\section{Robust Algorithm for Linear Regression} \label{sec:main-algorithm}

\subsection{Notation}
Before introducing our algorithm, we define the necessary notations in this subsection. Let $X$ be a random variable and $D$ be a distribution. We use $X\sim D$ to denote that $X$ is drawn from distribution $D$. 
For $S$ being a multiset of examples, we write $X\sim S$ to denote that $X$ is drawn 
uniformly at random from $S$. Given a multiset $S$ that contains samples $\{(X_i,y_i)\}$ drawn from distribution $D$, 
we use $\beta$ to denote $\E_{(X,y)\sim D}[yX]$, $\beta_{S}$ to denote $\E_{(X, y) \sim S}[yX]$, $M_S$ to denote $\E_{(X,y) \sim S}[{(yX-\beta)(yX-\beta)^T}]$ and $\widehat{{M}}_{S}$ to denote $\E_{(X,y) \sim S}[{(yX-\beta_{S})(yX-\beta_{S})^T}]$. Notice that $\beta,\beta_S,\widehat{{M}}_{S}$ each corresponds to the population mean, empirical (sample) mean, empirical (sample) covariance of $yX$. In our model, where $y = \beta^TX+\eta$ and $\E[XX^T]=I$, the expectation of $yX$ is equivalent to the weight vector $\beta$. Throughout this paper, for a vector $\beta$, we use $\|\beta\|$ or $\|\beta\|_2$ to denote the $\ell_2$-norm of the vector, 
and we use $\|M\|$ or $\|M\|_2$ to denote the operator norm of a matrix $M$. We use $\sigma^2$ to denote the variance 
of the noise $\eta$, and $\sigma_y^2$ to denote the variance of $y$, which is $\sigma^2+\|\beta\|^2$ in our setting. 

Under our corruption model where an $\eps$-fraction of the samples can be arbitrarily corrupted, 
we will typically use $S$ to denote the set of samples before being corrupted by the adversary. 
Given a set of samples, $S'$, we denote the set $E$ to be $S' \setminus S$ (which contains the samples added by the adversary), 
and the set $L$ to be $S\setminus S'$ (which contains the samples removed from the set of clean samples). 

\subsection{Basic Robust Linear Regression Algorithm}

The algorithm that achieves the performance guarantee stated in Theorem~\ref{thm:alg-lr-gaussian} 
is an iterative algorithm that invokes the following algorithm, Algorithm~\ref{alg:filter-LR-identity}, multiple times as a subroutine. 
Every time Algorithm~\ref{alg:filter-LR-identity}  gets called, it either returns an estimate of $\beta$ or returns a set of ``cleaner'' 
data points on which another iteration of Algorithm~\ref{alg:filter-LR-identity} can be invoked. We describe the algorithm as follows:
\begin{algorithm}%[htb]
\begin{algorithmic}[1]
\State \textbf{procedure}  \textsc{Filter-LR-Identity covariance}
\State \textbf{input:} A multiset $S'$ such that there exists an $(\eps,\tau)$-good $S$ with $\Delta(S, S') \le 2\eps$
\State \textbf{output:} Multiset $S''$ or mean vector $\beta_{S'}$ satisfying Proposition~\ref{prop:filter-lr-gaussian-known-cov}
\State Robustly estimate $\sigma_y$, the variance of $y$. 
Denote the estimation as $\sigma'_y$.
\State Compute the sample mean $\beta_{S'}=\E_{(X,y) \sim S'}[yX]$ 
and the sample covariance matrix $\widehat{M_{S'}} = \E_{(X,y) \sim S'}[{(yX-\beta_{S'})(yX-\beta_{S'})^T}]$. 
\State Compute approximations for the largest absolute eigenvalue of 
$\widehat{M_{S'}} - (\sigma'_y I+\beta_{S'}\beta_{S'}^T)$, 
$\lambda^{\ast} := \|\widehat{M_{S'}} - (\sigma'_y I+\beta_{S'}\beta_{S'}^T)\|_2,$ and the associated unit eigenvector $v^{\ast}.$

\State \textbf{if} { $\|\widehat{M_{S'}} - (\sigma'_y I+\beta_{S'}\beta_{S'}^T)\|_2 \leq O({\sigma_y}^2\eps\log^2(1/\eps))$}  
	\State \hspace{1cm} \textbf{return} $\beta_{S'}$.  \label{step:bal-small-G}
\State \textbf{end}
\State \label{step:bal-large-G}  Let $\delta := 3 \sqrt{ \eps  \lambda^*}/\sigma_y'.$ 
Find $T>0$ such that
$$
\Pr_{(X,y) \sim S'} \left[\frac{|v^{\ast} \cdot (yX-\beta^{S'})|}{\sigma'_y}>T+\delta \right] > 32\exp(-T/16)+8 \frac{\eps}{{T^2\log\left(N/\tau\right)}}.
$$
\State \label{step:gaussian-mean-filter} \textbf{return} the multiset 
$S''=\{(X,y) \in S' \mid  \frac{|v^{\ast} \cdot (yX-\beta_{S'})|}{\sigma'_y}  \leq T+\delta\}$.
\end{algorithmic}
\caption{Filter algorithm for LR with identity covariance}
\label{alg:filter-LR-identity}
\end{algorithm}

The following proposition formalizes the guarantee of Algorithm~\ref{alg:filter-LR-identity} 
that it either returns a cleaner dataset or an estimate of $\beta$ with $\ell_2$-error 
at most $O(\sigma_y \eps \log(1/\eps))$, where $\sigma_y = \sqrt{\sigma^2+\|\beta\|^2}$.

\begin{proposition} \label{prop:filter-lr-gaussian-known-cov}
Let $G \sim \normal(0, I_d)$ and $\eps, \tau>0$.
Let $S$ be an $(\eps,\tau)$-good set with respect to $(G, \beta)$.
Let $S'$ be any multiset with $\Delta(S,S') \leq 2\eps$. The algorithm \textsc{Filter-LR-Identity covariance}
runs in polynomial time and, given $S'$ and $\eps>0,$ returns one of the following:
\begin{enumerate}
\item[(i)]  A vector $\widehat{\beta} \in \R^d$ such that $\|\widehat{\beta}-\beta \|_2 = O(\sigma_y \eps \log(1/\eps)),$
\item[(ii)] A multiset $S'' \subseteq S'$ such that $\Delta(S,S'') \new{<} \Delta(S,S')$,
\end{enumerate}
where $\Delta(S,S')$ is the size of the symmetric difference of multisets $S$ and $S'$ divided by the cardinality of $S.$
\end{proposition}

The proof of Proposition~\ref{prop:filter-lr-gaussian-known-cov} is deferred to Appendix~\ref{ssec:alg}. Assuming Proposition~\ref{prop:filter-lr-gaussian-known-cov} holds, we are now ready to show Theorem~\ref{thm:alg-lr-gaussian}, which is restated below for convenience:

\vspace{.5cm}\noindent \textbf{Theorem~\ref{thm:alg-lr-gaussian}} \emph{Let $S'$ be an $\eps$-corrupted set of labeled samples of size $\Omega((d/\eps^2) \polylog(\frac{d}{\eps \tau}))$.
There exists an efficient algorithm that on input $S'$ and $\eps>0$, returns a candidate vector $\widehat{\beta}$
such that with probability at least $1-\tau$ it holds $\|\widehat{\beta} - \beta\|_2  = O(\sigma_y \eps \log(1/\eps))$ (recalling that $\sigma_y=\sqrt{\sigma^2 +\|\beta\|_2^2}$).
}

\begin{proof}[Proof of Theorem~\ref{thm:alg-lr-gaussian}]
By the definition of $\Delta(S, S'),$ since $S'$ has been obtained from $S$
by corrupting an $\eps$-fraction of the points in $S,$ we have that
$\Delta(S, S') \le 2\eps.$ By Proposition~\ref{prop:good-set} (see below), 
the set $S$ of uncorrupted samples
is $(\eps,\tau)$-good with respect to $G$ with probability at least $1-\tau.$
We henceforth condition on this event.

%Since $S$ is $(\eps,\tau)$-good, all $x \in S$ 
%have $\|x-\mu^G\|_2 \leq O(\sqrt{d \log |S|/\tau})$. 
%Thus, the \textsc{NaivePrune} procedure does not remove from $S'$ any member of $S$. 
%Hence, its output, $S''$, has $\Delta(S, S'') \leq \Delta(S, S')$ 
%and for any $x \in S''$, there is a $y \in S$ with $\|x-y\|_2  \leq O(\sqrt{d \log |S|/\tau})$. 
%By the triangle inequality, for any $x,z  \in S''$, $\|x-z\|_2 \leq O(\sqrt{d \log |S|/\tau})= O(\sqrt{d \log (d/\eps\tau}))$.

We iteratively apply the \textsc{Filter-LR-Identity covariance} procedure 
of Proposition~\ref{prop:filter-lr-gaussian-known-cov} until it terminates
returning a vector $\beta$ with $\|\widehat{\beta}-\beta\|_2 = O(\sigma_y\eps\log(1/\eps)).$
We claim that we need at most \new{$O(N)$} iterations for this to happen, simply because the sequence of iterations results in a sequence of sets $S_i'$ satisfy $S_i'> S'_{i+1}$.
\end{proof}

To better illustrate how Algorithm~\ref{alg:filter-LR-identity} works, we provide a proof sketch 
of Proposition~\ref{prop:filter-lr-gaussian-known-cov}. 
Our algorithm succeeds under a set of deterministic conditions 
that are satisfied by an uncorrupted set of samples with high probability. 
%Notice that for a set of examples $S=\{X_1,y_1,\ldots,X_n,y_n\}$, we define $\beta_S=\E_{X,y\sim S}[yX]=\sum_{i=1}^n \frac{y_iX_i}{n}$ and $M_S = \E_{X,y\sim S}[(yX-\beta)(yX-\beta)^T] = \frac{1}{n}\sum_{i=1}^n(yX-\beta)(yX-\beta)^T$

\begin{definition} \label{def:good-set}
Let $G \sim \normal(0, I_d)$, $\beta \in \R^d$, and $\eps,\tau >0$.
We say that a multiset $S$ of elements in $\R^d \times \R$ is {\em $(\eps,\tau)$-good} 
\new{(with respect to $(G, \beta)$)}
if the following conditions are satisfied:
\begin{itemize}
\item[(i)] For all $(X, y) \in S$, we have $\|\frac{yX}{\sigma_y}\|_2 \leq 4\sqrt{d} \log (|S|/\tau)$ and $y/\sigma_y \le \sqrt{4\log(|S|/\tau)}$.
\item[(ii)]  For every $v\in \R^d$ with $\|v\|_2=1$,
we have that 
$$\Pr_{(X, y) \sim S}\left[ \frac{y(v\cdot X)}{\sigma_y}>  T \right] \le 16\exp(-T/16)+\frac{\epsilon}{T^2\log (N/\tau)} \;.$$

\item[(iii)]  We have that $\left\|{\beta_S - \beta} \right\|_2\leq O(\sigma_y\eps) .$

\item[(iv)]  We have that $\left\|M_S- (\sigma_y^2I+{\beta \beta^T}) \right\|_2 \leq O({\sigma_y^2 \eps}).$
\end{itemize}
\end{definition}
Roughly speaking, condition (i) claims that none of the uncorrupted samples is too big in magnitude, condition (ii) establishes the empirical tail bound of the set of samples, condition (iii) and (iv) guarantee that the empirical mean and empirical covariance converge well to the true mean and covariance. The following proposition states that the above deterministic properties 
hold with high probability for a set of samples of near-linear size:

\begin{proposition} \label{prop:good-set}
Let $G \sim \normal(0, I_d)$ and $\eps, \tau>0$.
If the multiset $S$ consists of {$N = \Omega( (d/\eps^2) \poly\log(d/\eps\tau))$} labeled samples 
$(X, y) \in \R^d \times \R$, where $X \sim G$ and $y=\beta^T X {+\eta}$, {where $\eta \sim \normal(0,\sigma^2)$,}
then $S$ is $(\eps,\tau)$-good with respect to $(G, \beta)$ with probability at least $1-\tau$.
\end{proposition}

We note that the sample size in the above proposition is optimal, up to logarithmic factors, and is crucial
in establishing the near-optimal sample complexity of our algorithm.
The proof of Proposition~\ref{prop:good-set} is deferred to Appendix~\ref{sec:reg1}.

Given that the deterministic conditions hold for the uncorrupted data, our algorithm simply computes the sample mean and covariance (Step 5). Notice that condition (iii) and (iv) also establishes a connection between the sample mean and covariance, in the sense that for uncorrupted data, sample covariance can be predicted using the sample mean. Hence, the algorithm checks whether the sample mean and covariance satisfies their presumptive relationship, i.e., $\widehat{M_{S'}} \approx (\sigma'_y I+\beta_{S'}\beta_{S'}^T)$ (Steps 6, 7). 
If it is the case, the sample mean cannot possibly be corrupted by too much due to Corollary~\ref{cor:MApproxCor-G}, and hence the algorithm can output the sample mean confidently (Step 8). If it is not the case, the sample covariance must have been corrupted by a lot in some direction and thus violate the tail bound in this direction. The algorithm will then find a threshold such that there are more samples beyond the threshold than twice of the number predicated by condition (iii) (Step 10), and remove all the sample beyond the threshold (Step 11), which contains more ``bad'' samples than uncorrupted samples, due to Claim~\ref{claim:filter}. The full proof of 
Proposition~\ref{prop:filter-lr-gaussian-known-cov} is given in Appendix~\ref{ssec:alg}.

\subsection{Main Algorithm: Removing the Dependence on $\|\beta\|_2$ in the Error}

In this section, we describe an algorithm establishing Theorem~\ref{thm:alg-lr-gaussian-2}, 
which we restate below for completeness.

\vspace{.3cm}\noindent \textbf{Theorem~\ref{thm:alg-lr-gaussian-2}} \emph{Let $S'$ be an $\eps$-corrupted set of labeled samples of size $\Omega((d/\eps^2) \polylog(\frac{d}{\eps \tau}))$.
There exists an efficient algorithm that on input $S'$ and $\eps>0$, returns a candidate vector $\widehat{\beta}$
such that with probability at least $1-\tau$ it holds $\|\widehat{\beta} - \beta\|_2  = O(\sigma \eps \log(1/\eps))$.
}

\medskip

Similarly to the basic algorithm of the previous subsection, the algorithm that achieves the performance 
guarantee of Theorem~\ref{thm:alg-lr-gaussian-2} is iterative and invokes 
Algorithm~\ref{alg:filter-LR-identity-2} multiple times as a subroutine. 
Every time Algorithm~\ref{alg:filter-LR-identity-2}  gets called, it either returns an estimate of $\beta$ or returns a set of ``cleaner'' 
data points.

\begin{algorithm}%[htb]
\begin{algorithmic}[1]
\State \textbf{procedure}  \textsc{Filter-LR-Identity covariance-2}%{$S',\eps,\tau$}
\State \textbf{input:} A multiset $S'$ such that there exists an $(\eps,\tau)$-representative $S$ with $\Delta(S, S') \le 2\eps$
\State \textbf{output:} Multiset $S''$ or mean vector $\beta_{S'}$ satisfying Proposition~\ref{prop:filter-lr-no-beta}.

\State Let $\beta'=\E_{(X,y) \sim S'}[X X^T]^{-1} \E_{S'}[yX]$.
\State Robustly estimate the standard deviation of $y-\beta' \cdot X$ as $\sigma'$ using its interquartile range.

\State \textbf{if} { $\E_{(X,y) \sim S'}[(y-\beta' \cdot X)^2] \geq (1+\Omega(\eps\log(1/\eps)))\sigma'$,} 

	\State \hspace{1cm}\label{step:gaussian-filter-1}
	Find $T$ such that
	$$\Pr_{(X,y) \sim S'}[|y-\beta' \cdot X| > T] \geq 15 \exp(-T^2/4\sigma'^2) + \frac{\eps^2 \sigma'^2}{T^2 \log^3(d|S|/\tau)} \; .$$

	\State \hspace{1cm} \textbf{return} $S''=\{(X,y) \in S'|: |y-\beta' \cdot X| \leq T\}$.

\State \textbf{end}
\State Let $U$ be the subset of $S'$ with $|y-\beta' \cdot X| >  6\sqrt{\ln(1/\eps)} \sigma'$.
\State Let $w=(1/\|w'\|_2) w'$, where  $w'=\E_{(X,y) \sim S' \setminus U}[(y-\beta' \cdot X) X]$.

\State \textbf{if} { $\E_{(X,y) \sim S'}[(w \cdot X)^2] \geq 1+ \Omega(\eps\log(1/\eps))$,} 

\State \hspace{1cm}Find $T$ such that \label{step:gaussian-filter-2}
$$\Pr_{(X,y) \sim S'}[|w \cdot X| > T] \geq 15 \exp(-T^2/4) + \frac{\eps^2}{T^2 \log^3(d|S|/\tau)} \; .$$
\State \hspace{1cm} \textbf{return} $S''=\{(X,y) \in S'| |w \cdot X| \leq T\}$.
\State \textbf{end}
\State \label{step:lin-reg-filter}Find the top eigenvalue $\lambda^{\ast}$, and corresponding eigenvector $v^{\ast}$, of $\Sigma=\E_{S'\setminus U}[(y-\beta' \cdot X)^2 X X^T] - \E_{S'\setminus U}[y^2 X X^T]$.
\State \textbf{if}{ ($\lambda^{\ast} \geq (1+\Omega(\eps \log^2(1/\eps))) \sigma'^2$)}

\State \hspace{1cm}$\delta := 3 \sqrt{ \eps  \lambda^*}/\sigma'.$  Find $T$ such that
$$\Pr_{S' \setminus U}[|(y-\beta' \cdot X)(v^{\ast} \cdot X)| > T + \delta ] \geq 72 \exp(-T/16\sigma') + \frac{\eps}{T^2 \log(d|S|/\tau)} \;.$$
\State \hspace{1cm} \textbf{return} $S''=U \cup \{(X,y) \in S' \setminus U | |(y-\beta' \cdot X)(v^{\ast} \cdot X)| > T \}$.
\State \textbf{end}
\State \label{step:at-last-2} \textbf{return} $\beta'$.
\end{algorithmic}
\caption{Filter algorithm for LR with identity covariance with no $\|\beta\|_2$ dependence.}
\label{alg:filter-LR-identity-2}
\end{algorithm}

More specifically,
in order to estimate $\beta$ within $\ell_2$-error $O(\sigma \eps \log(1/\eps))$, as stated in Theorem~\ref{thm:alg-lr-gaussian-2}, we repeatedly execute Algorithm~\ref{alg:filter-LR-identity-2} until it returns an estimate of $\beta$. Our Algorithm~\ref{alg:filter-LR-identity-2}, which has no dependency on $\|\beta\|$, is a combination of three filter algorithms, which first computes the ordinary least square estimator $\beta'$, and then filters $y-\beta'\cdot X, X$ and $(y-\beta'\cdot X)X$ subsequently. If any of the filter algorithms removes samples, we can safely run the algorithm for the next iteration with the new set of samples, simply because our filter algorithm guarantees to remove more bad samples than good samples. If the set of samples passes all three filters, the sample mean must be close to the true mean, which means $\E_{X,y\sim S'}[(y-\beta'\cdot X)X] \approx \E [(y-\beta'\cdot X)X]$ due to the third filter. However, due to the property of the ordinary least squares estimator that $\E_{X,y\sim S'}[(y-\beta'\cdot X)X]=0$, we have $\E [(y-\beta'\cdot X)X] = \beta-\beta' \approx 0$. Specifically, we show that the magnitude of the error $ \|\beta-\beta'\| = \|\E [(y-\beta'\cdot X)X]-\E_{X,y\sim S'}[(y-\beta'\cdot X)X]\| =O(\eps \log(1/\eps) \sqrt{\sigma^2+\|\beta-\beta'\|^2})$, which in turn depends on $\|\beta-\beta'\|$. This immediately implies $\|\beta-\beta'\|=O(\eps \log(1/\eps)\sigma)$ for sufficiently small $\eps$. The guarantee of Algorithm~\ref{alg:filter-LR-identity-2} is stated below and the formal proof of correctness can be found in Appendix~\ref{sec:alg2-proof}.

\begin{proposition} \label{prop:filter-lr-no-beta}
Let $G \sim \normal(0, I_d)$ and $\eps, \tau>0$.
Let $S$ be $(\eps,\tau)$-representative with respect to $(G, \beta)$.
Let $S'$ be any multiset with $\Delta(S,S') \leq 2\eps$. There exists a polynomial time algorithm \textsc{Filter-LR-Identity covariance-2}
that, given $S'$ and $\eps>0,$ returns one of the following:
\begin{enumerate}
\item[(i)]  A vector $\widehat{\beta} \in \R^d$ such that $\|\widehat{\beta}-\beta \|_2 = O(\sigma \eps \log(1/\eps)),$
\item[(ii)] A multiset $S'' \subseteq S'$ such that $\Delta(S,S'') \new{<} \Delta(S,S')$,
\end{enumerate}
where $\Delta(S,S')$ is the size of the symmetric difference of multisets $S$ and $S'$ divided by the cardinality of $S.$
\end{proposition}

Like the basic algorithm we discussed in the previous subsection (which has dependency on $\|\beta\|$), 
the success of our new algorithm relies on the deterministic conditions which hold with high probability 
for a set of clean samples of size $N=\Omega(d\polylog(d/\eps\tau)/\eps^2)$. Definition~\ref{def:good-set-2} 
in Appendix~\ref{sec:good-set-2} lists the conditions that need to hold for our algorithm to work, which are similar to Definition~\ref{def:good-set} for the basic algorithm. As usual, each condition consists of four sub-conditions, which guarantee that the set of uncorrupted samples are bounded, satisfy certain tail bounds and have mean and covariance concentration. The first condition simply holds for a set of samples from an isotropic Gaussian distribution, through which our filter algorithm on $X$ can remove the corrupted samples. The second condition holds for $y-\beta'\cdot X$ with arbitrary $\beta'$, which allows us to run a 
filter algorithm on $y-\beta'\cdot X$ for any $\beta'$. While the first two conditions are in analogy to those in 
Definition~\ref{def:good-set} for the basic algorithm, the third condition is different. Specifically, the third condition 
does not hold unconditionally for a set of clean data with size $O(d\polylog(d/\eps\tau)/\eps^2)$, 
but only after conditioning on $|y-\beta'\cdot X|$ being not too big. In Appendix~\ref{sec:good-set-2}, we prove 
that these conditions will be satisfied with high probability after $O(d\polylog(d/\eps\tau)/\eps^2)$ samples.

\section{Statistical Query Lower Bounds} \label{sec:main-sq}
In this section, we formally describe our main lower bound result and provide a high-level proof sketch. 
Consider the joint distribution of $(X, y)$ in a linear regression problem without corruptions 
when the covariance of $X$ is unknown.  
Formally, let $Q$ be the distribution of $(X, y)$, where  $X\sim \mathcal{N}(0,\Sigma)$ for some unknown but bounded $\Sigma$, 
and $y$ conditioned on $X$ has $y|X \sim \beta^T X + \eta$, 
where $\beta$ is unknown and $\eta \sim \mathcal{N}(0,\sigma^2)$ for unknown but bounded $\sigma^2$. 
If we consider noise given by Huber's $\eps$-contamination model, then instead of seeing samples from $Q$, 
we observe samples from $Q'$, which is a mixture between $Q$ and a noise distribution, i.e., $Q' = (1-\eps)Q+\eps N$. 
Here we show that given statistical query access to $Q'$, we cannot approximate $\beta$ 
well without needing precision stronger than is possible with a strongly sub-quadratic number of samples:

\begin{theorem} \label{thm:SQ-lb}
No algorithm given statistical query access to $Q'$, defined as above with unknown noise 
and unknown variances $\frac{1}{2} I \preceq \Sigma \preceq I$ and $\sigma^2 \leq 1$, 
gives an output $\beta'$ with $\|\beta'-\beta\|_2 \leq o(\sqrt{\eps})$ on all instances
unless it uses more than  $2^{\Omega(d^c)} d^{4c-2}$ calls to the 
$$\mathrm{STAT}\left(O(d^{2c-1}) e^{O(1/\eps)} \right) \textrm{ or } \mathrm{VSTAT}\left( O(d^{2-4c})/e^{O(1/\eps)} \right)$$ 
oracles for any $c > 0$.
\end{theorem}

The detailed proof of Theorem~\ref{thm:SQ-lb} is given in Appendix~\ref{sec:sq-app}.

%\blue{Relate to information theoretic thing and explain why we want the lower bound on $\Sigma$.}
Informally speaking, the theorem shows that no Statistical Query algorithm can 
approximate $\beta$ to within $o(\sqrt{\eps})$ with fewer than exponential in $d$ queries, 
unless using queries of precision greater than we could simulate with $O(d^2/e^{O(1/\eps)})$ samples. 
Notice that without the lower bound on $\Sigma$, the result would be unsurprising. 
Indeed, if $\Sigma v=0$ for some non-zero $v$, then we could not approximate $v \cdot \beta$ at all, 
simply because $v \cdot \beta$ can be arbitrary without affecting $y$. 

In the proof of Theorem~\ref{thm:SQ-lb}, we use the construction in Proposition~3.3 of~\cite{DKS17-sq}, 
which intuitively says that if we have a distribution which is standard Gaussian in all except one direction, 
then if the low-degree moments match the standard Gaussian, then that direction is hard to find with an SQ algorithm. 
The idea is that, if we consider $X$ conditioned on $y$ for non-zero $\beta$, 
then $X$ has a non-zero mean in the $\beta$ direction. The conditional distribution $X|y$ is derived as follows:

\begin{lemma}\label{lem:ycondx} 
Let $Q$ be the joint distribution of $(X,y)$ with $X \sim \mathcal{N}(0,\Sigma)$ 
and $y|X \sim \beta^T X + \eta$, where $\beta$ is unknown an $\eta \sim \mathcal{N}(0,\sigma^2)$. 
Then $y \sim \mathcal{N}(0, \sigma_y^2)$, where $\sigma_y^2 = \beta^T \Sigma \beta+\sigma^2$ 
and $X|y \sim \mathcal{N}(\frac{y}{\sigma_y}\Sigma\beta , \Sigma-\frac{(\Sigma\beta)(\Sigma\beta)^T}{\sigma_y})$. 
\end{lemma}
\begin{proof} 
Notice that $(X,y)$ is a $d+1$ dimensional Gaussian distribution with covariance 
$$\begin{bmatrix}
\Sigma &\Sigma \beta\\
\beta^T \Sigma &\beta^T\Sigma \beta+\sigma^2
\end{bmatrix} \;.
$$
By the mean and covariance formula of the conditional distribution of a Gaussian, we have that 
$X|y\sim \mathcal{N}(\frac{ y}{\sigma_y}\Sigma\beta,\Sigma-\frac{\Sigma\beta\beta^T\Sigma}{\sigma_y})$.
\end{proof}

Notice that $X|y$ is indeed standard Gaussian in all except the $\Sigma\beta$ direction. 
By adding corruptions, we can make the distribution of $X|y$ projected onto $\Sigma\beta$ 
agree with the first three moments of $\mathcal{N}(0,I)$ and, like the construction of \cite{DKS17-sq}, 
still be a standard Gaussian in all the other orthogonal directions. Then we can show that we cannot 
find the direction of $\beta$ with an SQ algorithm. Lemma~\ref{lem:uvchi2} establishes the upper bound 
of the statistical correlation between a pair of distributions under our construction, which allows the 
classical statical query scheme (see, e.g., Corollary 3.12 in~\cite{Feldman13}) to be applied 
and yield the desired lower bound.  

The further the mean of $X$ conditioned on $y$ is from $0$, the more noise needs to be added to match the first three moments. Lemma~\ref{lem:A-properties}, which is the main lemma of the lower bound proof, shows that we can match the first three moments by adding $O(\mu^2)$ fraction of noise when the $X|y$ has mean $\mu$ in the $\beta$ idrection. As long as $\|\beta\|_2=O(\sqrt{\eps})$, after taking the integral over $y$, the overall noise added will still be smaller than $\eps$. 

\paragraph{Acknowledgements.} We would like to thank Jason Lee for his contributions to the early stages of this work.
I.D. and A.S. thank Daniel Kane for numerous discussions on robust high-dimensional estimation over the 
last five years.

\bibliographystyle{alpha}
\bibliography{allrefs}

\newpage

\appendix

\section{Proof of Proposition~\ref{prop:good-set}: Deterministic Regularity Conditions for Algorithm~\ref{alg:filter-LR-identity}}\label{sec:reg1}
\iffalse
We define the following sets:
\begin{enumerate}
\item $S$ = clean data
\item $S'$ = current set of samples = all data minus what is thrown away= (should be indexed by iteration)
\item $E$= corrupted data in $S'$ = $S' \setminus S$
\item $L = S\setminus S'$= any clean data in $S$ that is not in $S'$
\end{enumerate}
\fi

This section establishes Proposition~\ref{prop:good-set}.

\paragraph{Technical Facts.}
We will require a couple of technical facts.
We start with the following basic Gaussian concentration result:

\begin{fact} \label{fact:tail-bound} 
Let $G \sim \normal(\mu, I_d)$. Then for any unit vector $v \in \R^d$ we have that 
$\Pr_{X \sim G}\left[|v \cdot (X-\mu)| \geq T \right] \leq 2\exp(-t^2/2)$.
\end{fact}

We will make essential use of the following concentration inequality for quadratic forms:

\begin{lemma}[Hanson-Wright Inequality \cite{vershynin2010introduction}]
\label{lem:hanson-wright}
Let $X \sim \normal(0, I_d)$ and $A \in \R^{d \times d}$.
Then for some absolute constant $c_0$, for every $t \geq 0$,
$$\Pr\left(\left|X^TAX - \E[X^TAX]\right| > t\right) \leq 2 \exp\left(-c_0 \cdot \min \left(\frac{t^2}{\|A\|_F^2}, \frac{t}{\|A\|_2}\right)\right) \;.$$
\end{lemma}

Throughout this proof, we will use $S$ to denote a set of 
labeled samples $(X, y) \in \R^d \times \R$, where $X \sim G$ 
and $y=\beta^T X +\eta$, where $\eta \sim \normal(0,\sigma^2)$. We will use $N$ to denote 
the cardinality of $S$.

The following lemma proves property (i) of Definition~\ref{def:good-set}:

\begin{lemma} \label{lem:random-good-gaussian-mean}
Let $G \sim \normal(0, I_d)$ and $\eps, \tau>0$.
If the multiset $S$ consists of {$\Omega( (d/\eps^2) \poly\log(d/\eps\tau))$} pairs $(X,y)$, 
where $X \sim G$ and $y=\beta \cdot X {+\sigma E}$,
it satisfies $\|\frac{yX}{\sigma_y}\|_2 \leq 4\sqrt{d} \log (|S|/\tau)$, $\|X\|_2\le 2\sqrt{d\log(|S|/\tau)}$, $y/\sigma_y\le 2\sqrt{\log(|S|/\tau)}$ with probability at least $1-\tau$.
\end{lemma}
\begin{proof}
Let $N = \Omega( (d/\eps^2) \poly\log(d/\eps\tau))$ be the size of $S$.
Consider the unlabeled set of samples $X_1, \ldots, X_N$ drawn from $G$.To establish (i), we note that the probability that a coordinate of a sample $X_i$ has absolute value at least $\sqrt{2\log(20Nd/\tau)}$ 
is at most $\tau/(10dN)$ by Fact~\ref{fact:tail-bound}. By a union bound over $d$ coordinates, 
the probability that all coordinates of all samples have absolute value smaller than $\sqrt{2\log(20Nd/\tau)}$ 
is at least $1-\tau/10$. In this case, $\|X\|_2 \leq \sqrt{ 2 d \log(20Nd/\tau)} \le  2\sqrt{d \log(N/\tau)}$ assuming $N>20d$.
Also note that $y/\sigma_y \sim \normal(0, 1)$. 
By Fact~\ref{fact:tail-bound}, the probability that $|y|/\sigma_y \ge \sqrt{2\log(20N/\tau)}$ is at most $\tau/(10N)$. 
By a union bound, the probability that all $|y|/\sigma_y$ are smaller than $\sqrt{2\log(20N/\tau)}<\sqrt{4\log(N/\tau)}$ is at least $1-\tau/10$. 
Hence, $\|\frac{yX}{\sigma_y}\|_2 \le (4\sqrt{d}  \log(N/\tau))$ holds for all the samples with probability at least $1-\tau/10$. 
This completes the proof.
\end{proof}

We will require the following technical claim:
\begin{claim}\label{claim:good-set-tail}
\new{Let $S$ be of size at least a sufficiently large multiple of $d\log^4(d/(\delta\tau))/\eps^2$}. 
With probability at least $1-\tau/10$, we have that for any unit vector $v \in \R^d$ and  $T>0$ it holds
$$\Pr_{(X, y) \sim S}\left[|v\cdot X| > T\right] \leq 5 \exp(-T^2/4)+\frac{\eps^2}{{T^2\log^3\left(|S|/\tau\right)}} \;.$$
\end{claim}

\begin{proof}
We start with the following claim:
\begin{claim} \label{clm:single-v} 
Let $S$ be a set of $N \geq 10$ independent samples from $\normal(0,1)$. 
For  $0 < \delta \leq 1$, we have with probability at least $1-\ln(N/\delta) \exp(-\Omega(N\delta/\log(1/\delta)))$, for all $T$,
$\Pr_{X \sim S}[|X| \geq T] \leq 5 \exp(-T^2/2) + \delta/T^2$ and all $X \in S$ have $|X| \leq 2\sqrt{N/\delta}$.
\end{claim}
\begin{proof}
To prove the claimed tail for all $T$, we first show that for any $T=2^i$, where $1 \leq i \leq \ln(N/\delta)$, 
the claimed tail bound divided by $2$ (i.e., $\frac{5}{2} \exp(-T^2/2) + \frac{\delta}{2T^2}$) 
holds with probability $1-\exp(-\Omega(N\delta/\log(1/\delta))$. 
Then, by a simple union bound, we get the desired upper bound for all $T$.

Given $Y \sim \normal(0,1)$, we have that $\Pr[|Y| \geq T] = 2\erfc(T)$ where $\erfc$ is the complementary error function. 
We define the function $Q(T)=5 \erfc(T)/2 + \delta/2T^2$. 
Observe that $N\Pr_{X\sim S}[|X| \geq T]$ is a sum of $N$ independent Bernoulli random variables with mean $2\erfc(T)$.
Since $Q(T) \geq \frac{5}{4} (2\erfc(T))$, by the Chernoff bound, $\Pr_{X\sim S}[|X| \geq T] \geq Q(T)$ 
with probability at most $\exp(-N Q(T)/60)$. Let $T'$ be such that $\erfc(T')=\delta^2/4T'^4$. 
Since $\exp(-T^2/2)/T \leq \erfc(T) \leq \exp(-T^2/2)$ for all $T > 0$, 
we have that $T'^2/2 = \Theta(\ln(T') + \ln(1/\delta))$ and hence $T'=\Theta( \sqrt{\ln(1/\delta)})$. 
Thus, we have that for all $T \leq T'$, $Q(T) \geq \delta/2T'^2 = \Omega( \delta/\log(1/\delta))$ and this bound suffices.

When $T \geq T'$, note that $\erfc(T) \leq (\delta/2T^2)^2$.
Here we need to use a more explicit version of the Chernoff bound. 
That gives that $\Pr[|X| \geq T] \geq Q(T)$ with probability at most $\exp(-ND(Q(T)||2\erfc(2)))$, 
where $D(p||q)=p\ln(p/q) + (1-p) \ln((1-p)/(1-q))$ is the KL-divergence between Bernoulli's 
with probabilities $p$ and $q$. When $T \geq T'$, $p=\delta/2T^2$, $q=2\erfc(T)$, we obtain
\begin{align*}
D(Q(T)||q) & \geq D(p||q) = p\ln(p/q) + (1-p) \ln((1-p)/(1-q)) \\
& \geq p\ln(p/q) - \ln(1-p) \\
& \geq p (\ln(p/q) - 1 -O(p)) \\
& = (\delta/2T^2) (\ln(p/2\erfc(T))  - 1 - O(\delta)) \\
& \geq  (\delta/2T^2) (\ln(p/2\erfc(T))  - 1 - O(\delta)) \\
& \geq (\delta/2T^2) (\ln(1/\erfc(T)/2  - 1 - O(\delta)) \\
& \geq (\delta/2T^2) (\ln(\exp(T^2/2))/2  - 1 - O(\delta)) \\
& \geq (\delta/2T^2) (T^2/4  - 1 - O(\delta)) \\
& \geq \delta/10 \;.
\end{align*}
Thus, we have that $\Pr[|v \cdot X| \geq T] \geq Q(T)$ with probability 
at most $\exp(-ND(Q(T)||2\erfc(2)))=\exp(-\Omega(N\delta))$ in this case. 

Note that $Q(T) < 1/N$ for $T \geq \max\{\sqrt{N/\delta}, \sqrt{2\ln (5N/2)}\}=\sqrt{N/\delta}$ for $N \geq 10$.
By a union bound, we have that for $T'=2^i$ for integers $1 \leq i \leq \ln(N/\delta)/2+1$ 
that $\Pr[|v \cdot X| \geq T'] \leq Q(T') \leq  5 \exp(T^2/2)/2 + \delta/2T^2$ for all such $T'$ is $O(\ln(N/\delta) \exp(-\Omega(N\delta))$.
Note that the largest $T'$ has $\Pr[|v \cdot X| \geq T'] < 1/N$ and since $X \sim S$ and $|S|=N$, this means that $\Pr[|v \cdot X| \geq T']=0$. 

If $T <1$, $5 \exp(T^2/2)/2 + \delta/2T^2 \geq 1$ and so the result is trivial. If $T \geq 2\sqrt{N/\delta}$, then  
$\Pr[|v \cdot X| \geq T]=0$. Otherwise, there is a $T'$ with $T/2 \leq T' \leq T$ 
and thus $\Pr[|v \cdot X| \geq T']  \leq \Pr[|v \cdot X| \geq T] \leq 5\exp(T^2/2) + \delta/T^2$.
This completes the proof of Claim~\ref{clm:single-v}.
\end{proof}

Let $C$ be a $1/5$-cover of the set of unit vectors including all coordinate directions of size $2^{O(d)}$. 
Then, by a union bound, Claim~\ref{clm:single-v} holds for $v \cdot X$ for all $v \in C$ 
except with probability $2^d \ln(N/\delta) \exp(-\Omega(N\delta/\log(1/\delta)))$, 
where $\delta=\frac{\eps^2}{2\log^3\left(|S|/\tau\right)}$. This is smaller than $\tau/10$ 
when $\Omega(N\delta/\log(1/\delta)) \geq \ln \ln(N/\delta) + d + \ln(1/\tau)$, 
which holds when $N$ is a sufficiently large multiple of $d \log(d/(\delta\tau))/\delta$. The latter statement in turn 
holds when $N$ is a sufficiently large multiple of $d\log^4(d/(\delta\tau))/\eps^2$. 
We assume this holds in the following. 

Now consider a unit vector $v$ and $T \geq 1$. 
Let $v_1=v$ and $T_i=T$. Let $v'_i \in C$ have $\|v_i-v'_i\|_2 \leq 1/5$. 
Then let $v_{i+1}=(v_i-v'_i)/\|v_i-v'_i\|_2$. 
If $x$ has $|v_i \cdot x| \geq T_i$, using the triangle inequality, it follows that 
either $|v'_i \cdot x| \geq T_i/2$ or else $|v_{i+1} \cdot x|= (v_i-v'_i)/\|v_i-v'_i\|_2 \cdot x \geq T_i/2\|v_i-v'_i\|_2 \geq 2T_i$. 
So we let $T_{i+1}=2T_i$ and we have
$$\Pr[|v_i \cdot X| \geq T_i] \leq \Pr[|v'_i \cdot X| \geq T_i/2]+ \Pr[|v'_{i+1} \cdot X| \geq T_{i+1}] \;.$$
If we iterate this procedure, for large enough $i$, we will have $T_i \geq 2 \sqrt{dN/\delta}$, 
and so $\Pr[|v'_i \cdot X| \geq T_i]=0$ and hence we have $\Pr[|v \cdot X| \geq T] \leq \sum_{i=1}^{\log(2 \sqrt{dN/\delta})}  \Pr[|v'_i \cdot X| \geq T_i/2].$ 
By Claim~\ref{clm:single-v}, we have  $\Pr[|v'_i \cdot X| \geq 2^{i-1} T] \leq 5\exp(-T^2 2^{2i-3}) + 2^{-(2i-2)}\delta/T^2 \leq 2^{-(2i-1)}(5\exp(-T^2/4) + 2\delta/T^2)$, 
and so
\begin{align*}
 \Pr[|v \cdot X| \geq T] & \leq \sum_i  \Pr[|v'_i \cdot X| \geq T_i/2] \\
 & \leq \sum_i 2^{-(2i-1)}(5\exp(-T^2) + 2\delta/T^2)) \\
 & \leq 5\exp(-T^2/4) + 2\delta/T^2 \\
 & \leq 5\exp(T^2/4) + \frac{\eps^2}{T^2\log^3\left(|S|/\tau\right)} \;.
\end{align*}
So we have shown that, for all $v$ and $T \geq 1$, 
$\Pr[|v \cdot X| \geq T] \leq 5\exp(T^2/4) + \frac{\eps^2}{T^2\log^3\left(|S|/\tau\right)}$. 
Since this is trivial for $T \leq 1$, we are done.
\end{proof}

The following lemma proves property (ii) of Definition~\ref{def:good-set}:

\begin{lemma} \label{lem:good-set-two}
For all $T>0$, we have that $\Pr_{(X, y) \sim S} \left[ |y (v\cdot X)/\sigma_y|>  T \right] \leq 8\exp(-T/8)+\frac{\epsilon}{T^2\log(N/\tau)}$.
\end{lemma}
\begin{proof} The proof of the lemma will make essential use of the following elementary fact:
\begin{fact}\label{fact:product-bound} 
For any pair of real random variables $A, B$, integers $a, b \in \Z$, and $T \in \R$, it holds:
$$\Pr[|A B| \geq T] \le \sum_{i=a}^{b} \Pr[(|A| \geq 2^i) \wedge (|B| \geq T/2^{i+1})] + \Pr[|A| \geq 2^b]+ \Pr[|B| \geq T/2^a] \;.$$
\end{fact}

Notice that when $T<16$, the RHS is greater than $1$ and hence the inequality is trivial. When $T>4\sqrt{d} \log(N/\tau)$, by condition (i) of Definition~\ref{def:good-set}, the LHS is $0$ which makes the inequality trivial as well. For the rest of the analysis we will assume $4\sqrt{d} \log(N/\tau)>T>16$.
By Fact~\ref{fact:product-bound}, we can write that 
$$\Pr_{(X,y) \sim S}[|y(v\cdot X)|/\sigma_y>  T] \leq 
\sum_{i=0}^{t} \Pr_{(X,y) \sim S}[(|y/\sigma_y| \geq 2^i) \wedge (|v \cdot X| \geq T/2^{i+1})] 
+ \Pr_{(X,y) \sim S} [|y/\sigma_y| \geq 2^t] + \Pr_{(X,y) \sim S}[|v\cdot X| \geq T] \;.$$ 
By condition (i) of Definition~\ref{def:good-set}, we have that 
$\Pr_{(X,y) \sim S}[|y/\sigma_y| \geq \sqrt{4\log N/\tau}]=0$. 

Thus, we can set the integer parameter $t$ to be 
$\min \left\{ \lceil \log_2 T \rceil , \lceil \log_2 ( \sqrt{4\log N/\tau}) \rceil \right\}$. 
By Claim~\ref{claim:good-set-tail}, the term $\Pr_{(X,y) \sim S}[|v\cdot X| \geq T]$ is at most 
$5 \exp(-T^2/4)+\frac{\epsilon}{T^2 \log^3(N/\tau))}$. 
Due to the simple fact 
$$\Pr_{(X,y) \sim S}[(|y/\sigma_y| \geq 2^i) \wedge (|v\cdot X| \geq T/2^{i+1})] 
\le \min \left\{ \Pr_{(X,y) \sim S}[|y/\sigma_y| \geq 2^i], \Pr_{(X,y) \sim S}[|v\cdot X| \geq T/2^{i+1}] \right\} \;,$$ 
we have 
\begin{align*}
&\sum_{i=0}^{t} \Pr_{(X,y) \sim S} [(|y/\sigma_y| \geq 2^i) \wedge (||v\cdot X \geq T/2^{i+1})] \\
&\leq \sum_{i=0}^{t} \min \left\{ \Pr_{(X, y) \sim S} [|y/\sigma_y| \geq 2^i] , \Pr_{(X, y) \sim S} [|v \cdot X| \geq T/2^{i+1}] \right\} \\
%=\sum_{i=0}^{t} \min(\exp(-2^{2i-3})+\frac{\epsilon}{2^{2i-2}\sqrt{d}}, \exp(-\frac{T^2}{2^{2i+3}})+\frac{\epsilon 2^{2i+2}}{T^2\log d})\\
%\le \log \log N/\tau \max_{a\in [1,\sqrt{\log N/\tau}]} \left( \min\left (\exp(-a^{2}/8)+\frac{4\epsilon}{a^{2}\sqrt{d}}, \exp(-\frac{T^2}{8a^2})+\frac{\epsilon 4a^2}{T^2\log d}\right )\right).
&\le \sum_{a = 1,2,2^2,\ldots,\sqrt{4\log (N/\tau)}} \left( \min\left(5\exp(-a^{2}/4)+ 
\frac{\eps}{a^{2}\log^3(N/\tau)}, 5\exp\left(-\frac{T^2}{16a^2}\right)+\frac{\eps 4a^2}{T^2 \log^3(N/\tau)}\right )\right) \;.
\end{align*}  
We first establish an upper bound for $T\ge 8\log N/\tau$ 
in which case $a^2\le T/2$. Each term in the summation above satisfies
\begin{align*}
\le 5\exp\left(-\frac{T^2}{16a^2}\right)+\frac{4\epsilon a^2}{T^2\log^3(N/\tau)}\le 5\exp(-T/8)+\frac{16\epsilon}{T^2\log^2(N/\tau)}.
\end{align*}
Note that there exists a sufficiently large universal 
constant $C>0$ such that for $N>C/\eps$, we have that 
$5\exp(-T/8)< \frac{16\eps}{100T^2\log^2(N/\tau)}$, 
which implies that the exponential term is negligible for this range of $T$. 
There are at most $\lceil \log_2 ( \sqrt{4\log N/\tau}) \rceil+1$ such terms in the summation, 
which yields an upper bound of the summation as $1.01(\lceil \log_2 ( \sqrt{4\log N/\tau}) \rceil+1)\frac{16\epsilon }{T^2\log^2(N/\tau)} \le \exp(-T/4)+\frac{\eps}{T^2\log(N/\tau)}$ for sufficiently large $N$.

We now prove an upper bound for the case that 
$T\le 8\log N/\tau$. For $a=\sqrt{T/2}$, the first term in the min function 
satisfies $5\exp(-a^{2}/4)+\frac{\epsilon}{a^{2}\log^3(N/\tau)}=5\exp(-T/8)+\frac{2\epsilon}{T\log^3(N/\tau)}\le 5\exp(-T/8)+\frac{16\epsilon}{T^2\log^2(N/\tau)}$. 
Similarly, the second term in the min function satisfies 
$5\exp(-\frac{T^2}{16a^2})+\frac{\epsilon 4a^2}{T^2\log^3(N/\tau)}\le 5\exp(-T/8)+\frac{16\epsilon}{T^2\log^2(N/\tau)}$. 
The first term monotonically decreases as $a$ gets larger. 
The second term monotonically increases as $a$ gets larger. 
We can thus conclude that
\begin{align*}
&\min_{1\leq a \leq T} \left (5\exp(-a^{2}/4)+ 
\frac{\eps}{a^{2}\log^3(N/\tau)}, 5\exp\left(-\frac{T^2}{16a^2}\right)+\frac{\eps 4a^2}{T^2 \log^3(N/\tau)}\right ) \\
&\leq 
5\exp(-T/8)+\frac{16\eps}{T^2 \log^2(N/\tau)} \;.
\end{align*}
There are at most $\lceil \log_2 T\rceil +1$ such terms, 
which yields an upper bound of
$$
5(\lceil \log_2 T \rceil+1) \exp(-T/8)+\frac{16\epsilon ( \lceil \log_2 T\rceil+1) }{T^2 \log^2(N/\tau)} \;.
$$
Now note that for $T\ge 16$, the above is smaller 
than $16\exp(-T/16)+\frac{\epsilon}{T^2 \log(N/\tau)}$ because we can assume $T\le 4\sqrt{d}\log(N/\tau)$. Hence, we obtain an upper bound of $16\exp(-T/16)+\frac{\epsilon}{T^2 \log(N/\tau)}$, 
which completes the proof.
\end{proof}

The following lemma proves property (iii) of Definition~\ref{def:good-set}:

\begin{lemma} \label{lem:betaS-error}
{For $N=\Omega(\frac{d}{\epsilon^2} \poly\log(d/\eps\tau))$},
we have that
$\|\beta_S-\beta\|_2\le \epsilon \sigma_y$,
with probability at least $1-\tau/10$.
\end{lemma}
\begin{proof}
The proof requires two simple technical claims.
Our first claim gives an explicit formula for the distribution of the projections of $yX$ in any direction:
\begin{claim} \label{claim:dec-chisq} 
For any $v \in \R^d$ with $\|v\|_2=1$, 
we have that 
$$(v\cdot X)y= \left(\frac{v\cdot \beta+\sigma_y}{2} \right) Z_1^2+\left(\frac{v\cdot \beta -\sigma_y}{2}\right) Z_2^2 \;,$$ 
where $Z_1,Z_2 \sim \normal(0,1)$ and $Z_1, Z_2$ are independent. 
\end{claim}
\begin{proof}
Let $Y_1 = v^T X, Y_2 =  \frac{(X- (v^T X) v)^T \beta +  \eps}{\sqrt{\sigma_y-v^T \beta}}$. 
Note that $Y_1,Y_2\sim \normal(0,1)$ and are independent. 
Define $ Z_1 = \sqrt{\frac{\sigma_y+v^T \beta}{2\sigma_y}}Y_1+\sqrt{\frac{\sigma_y-v^T \beta}{2\sigma_y}}Y_2, Z_2 = -\sqrt{\frac{\sigma_y-v^T \beta}{2\sigma_y}}Y_1+\sqrt{\frac{\sigma_y+v^T \beta}{2\sigma_y}}Y_2$. 
It is now easy to verify that the claim statement holds.
\end{proof}

\noindent Our second claim gives a tight concentration inequality for $\beta_S$:
\begin{claim}\label{claim:mean-conc}
For any $v \in \R^d$ with $\|v\|_2=1$ and $t>0$, we have that for some absolute constant $c_0$
$$\Pr [|v \cdot (\beta_S-\beta) | \new{>} t] \leq 2\exp\left(-c_0 \left(\min\left(\frac{4N^2t^2}{\|A\|_F^2},\frac{2Nt}{\|A\|_2}\right)\right)\right) \;,$$
where $\|A\|_F^2 = 2N ((v\cdot\beta)^2+\sigma_y^2)$ and 
$\|A\|_2=\max(|v\cdot \beta-\sigma_y|,|v\cdot \beta+\sigma_y|)$.
\end{claim}
\begin{proof}
Let $Z=(Z_1, Z_2, \ldots, Z_{2N})$ be a $2N$-dimensional vector, 
where each $Z_i$ is independently drawn from $\normal(0,1)$. 
Let $A \in \R^{2N \times 2N}$ be a diagonal matrix whose diagonal entries $A_{i, i}$
equal ${v\cdot \beta+\sigma_y}$ for $1 \leq i \leq N$, and ${v\cdot \beta-\sigma_y}$ for $N+1 \leq i \leq 2N$. 
Using Claim~\ref{claim:dec-chisq}, $v\cdot \beta_S$ can be written as 
$v \cdot \beta_S = (1/(2N)) Z^T A Z$. Applying the Hanson-Wright inequality (Lemma~\ref{lem:hanson-wright}),
we get  
\begin{align*}
\Pr[|v\cdot (\beta_S-\beta) | > t]= \Pr[|Z^TAZ-\E[Z^TAZ]| > 2Nt]\\
\le 2\exp\left(-\Omega\left(\min\left(\frac{4N^2t^2}{\|A\|_F^2},\frac{2Nt}{\|A\|_2}\right)\right)\right) \;,
\end{align*}
where $\|A\|_F^2 = 2N ((v\cdot\beta)^2+\sigma_y^2)$ and $\|A\|_2=\max(|v\cdot \beta-\sigma_y|,|v\cdot \beta+\sigma_y|)$.
This completes the proof.
\end{proof}

{We now have the necessary ingredients to prove Lemma~\ref{lem:betaS-error}.}
Let $t=\frac{\epsilon}{\sqrt{d}}\sigma_y$. 
Since $((v\cdot\beta)^2+\sigma_y^2) \le 2\sigma_y^2$ and 
$\max(|v\cdot \beta-\sigma_y|,|v\cdot \beta+\sigma_y|)\le 2\sigma_y$, 
we have $\min(\frac{4N^2t^2}{\|A\|_F^2},\frac{2Nt}{\|A\|})\ge \min(\frac{2N\epsilon^2}{d},\frac{N\epsilon}{\sqrt{d}})$. 
Given that $N=\Omega(\frac{d}{\epsilon^2} \poly\log(d/\eps\tau))$ and Claim~\ref{claim:mean-conc}, 
we have $\Pr[|v\cdot (\beta_S-\beta) |\le \frac{\epsilon}{\sqrt{d}}\sigma_y] \le \exp(-\log(10d/\tau))\le \frac{\tau}{10d}$. 
Taking a union bound over $d$ unit basis vectors of $\R^d$, 
we get that with probability at least $1-\frac{\tau}{10}$, 
no coordinate of $\beta_S-\beta$ has magnitude larger than 
$\frac{\epsilon}{\sqrt{d}}\sigma_y$, and hence $\|\beta_S-\beta\|_2 \le \eps \sigma_y$.
\end{proof}

{The following lemma proves property (iv) of Definition~\ref{def:good-set}:}

\begin{lemma} \label{lem:good-set-four}
{For $N=\Omega(\frac{d}{\epsilon^2} \poly\log(d/\eps\tau))$}, we have that
$\left\|M_S- (\sigma_y^2I+{\beta \beta^T}) \right\|_2 \leq {\sigma_y^2 \eps}$, 
with probability at least $1-\tau/10$.
\end{lemma}
\begin{proof}
The proof idea is the following: By standard results, it is straightforward to handle 
the concentration of the empirical covariance of a distribution with bounded support. 
Hence, we split the distribution into two parts. One part contains most of the probability mass 
and has almost identical covariance as the original distribution. In addition it has bounded support. 
The other part is unbounded but has small probability. We first argue that removing the second part 
has little effect on the covariance of the distribution. Then the empirical covariance concentration result follows easily.

Let $D$ denote the distribution of $yX$, and $D'$ be the distribution of $yX$ conditional on the event
$\cale =\{\|yX\|_2 \le \sigma_y 4\sqrt{d} \log (|S|/\tau)\}$. In this part of the proof, we argue that the covariance of the $D'$ is similar to the $D$. 
Denote the matrix $\E_{(X,y) \sim D'} [(yX)(yX)^T]$ as $F'$ and the matrix 
$\E_{(X,y) \sim D} [(yX)(yX)^T] = \sigma_y I + 2\beta\beta^T$ as $F$. 
For any unit vector $v$, we have that
\begin{align}
v^TFv&= \Pr_{X,y\sim D}[\cale] \E_{X,y\sim D}[(yv\cdot X)^2|\cale]+ 
\Pr_{X,y\sim D}[\cale^c] \E_{X,y\sim D}[(yv\cdot X)^2| \cale^c]\nonumber\\
&= \Pr_{X,y\sim D}[\cale] v^T F'v+ \Pr_{X,y\sim D}[\cale^c] \E_{X,y\sim D}[(yv\cdot X)^2| \cale^c]\label{eqn:vSpv}
\end{align}
By the argument in the proof of Lemma~\ref{lem:random-good-gaussian-mean}, we have that 
$\Pr[\cale^c] \le \frac{\tau}{|S|}$, the term $\Pr_{X,y\sim D}[\cale^c] \E_{X,y\sim D}[(yv\cdot X)^2|\cale^c]$ 
is less than the second moment of the $\frac{\tau}{|S|}$-tail of $y(v\cdot X)$ 
of which we are going to obtain an upper bound. 
Let $t'$ satisfy $\Pr[|y(v\cdot X)|>t']= \frac{\tau}{|S|}$. 
The second moment of the $\frac{\tau}{|S|}$-tail is by definition 
\begin{align}\label{eqn:cov-concentration1}
t'^2\frac{\tau}{|S|}+\int_{t'}^\infty \Pr[|y(v\cdot X)|>T] T dT \;.
\end{align}
Applying Claim~\ref{claim:mean-conc} with $N=1$, 
we get $\Pr[|(v\cdot X)y-v\cdot\beta|\ge t]\le 2\exp(-c_0(\frac{t}{\sigma_y}))$ 
for some absolute constant $c_0$. Notice that $\Pr[|y(v\cdot X)|>t+|v\cdot \beta|]\le \Pr[|y(v\cdot X)-v\cdot \beta|>t]$,
which is equivalent to $\Pr[|(v\cdot X)y|\ge t]\le \exp(-c_0(\frac{t-|v\cdot \beta|}{\sigma_y})) \le \exp(-c_0(\frac{t}{\sigma_y}-1))$. 
Pick appropriate $t=\Theta(\sigma_y \log(\frac{|S|}{\tau}))$ such that 
$\Pr[|y(v\cdot X)|>t]\le 2\exp(-\Omega(\log(\frac{|S|}{\tau})))= \frac{\tau}{|S|}$. 
Notice that $t\ge t'$ and we can rewrite Equation~(\ref{eqn:cov-concentration1}) 
as
\begin{align*}
&t'^2\frac{\tau}{|S|}+\int_{t'}^t \Pr[|y(v\cdot X)|>T]TdT + \int_{t}^\infty \Pr[|y(v\cdot X)|>T] T dT\\
\le& t^2\frac{\tau}{|S|}+\int_{t}^\infty \Pr[|y(v\cdot X)|>T]TdT \\
\le& \Theta(\sigma_y^2 \log^2(\frac{|S|}{\tau})\frac{\tau}{|S|})+\int_{t}^{\infty} \exp(-\Omega(\frac{T}{\sigma_y})) T dT \\
=& \Theta(\sigma_y^2 \log^2(\frac{|S|}{\tau})\frac{\tau}{|S|})+\sigma_y^2 \exp(-\Omega(\frac{t}{\sigma_y})) \frac{t}{\sigma_y} \\
=& \Theta(\sigma_y^2 \log^2(\frac{|S|}{\tau})\frac{\tau}{|S|})+O(\sigma_y^2 \log(\frac{|S|}{\tau}) \frac{\tau}{S})\\
=& O(\sigma_y^2 \log^2(\frac{|S|}{\tau})\frac{\tau}{|S|}) \;.
\end{align*}
By Equation~(\ref{eqn:vSpv}), 
$vF'v=vFv+ \sigma_y^2 O(\log^2(\frac{|S|}{\tau})\frac{\tau}{|S|}+\frac{\tau}{|S|}) 
= vFv+ \sigma_y^2 O(\log^2(\frac{|S|}{\tau})\frac{\tau}{|S|})$
and hence $\|F-F'\|_2= \sigma_y^2 O(\log^2(\frac{|S|}{\tau})\frac{\tau}{|S|})$. 

Now we argue that, with high probability, the empirical covariance concentrates to $F'$. 
The sampling procedure from $D'$ can be thought as the following: 
draw $N$ i.i.d samples from $D$, if $\|yX\|_2 \le \sigma_y 4\sqrt{d}\log (|S|/\tau)$
is satisfied for all samples, return the samples, otherwise declare failure. 
Hence, with probability $1-\frac{1}{10\tau}$, the samples in $S$ can be seen 
as distributed as $D'$. If we set $t=\sqrt{\frac{\log \frac{1}{\tau}}{\log d}}$ in 
Corollary 5.52 of \cite{vershynin2010introduction}, with probability at least $1-\tau/10$, 
$\|\E_{X,y\sim S}[(yX)(yX)^T]-F'\|_2 \le \epsilon \sigma_y^2$. 

Finally, we have
\begin{align}
\|M_S-(\sigma_y^2 I + \beta\beta^T)\|_2 
&= \|\E_{X,y\sim S}[(yX-\beta)(yX-\beta)^T]-(\sigma_y^2I+\beta\beta^T)\|_2 \\
&= \E_{X,y\sim S}[(yX)(yX)^T]-(\beta_S-\beta)\beta^T-\beta(\beta_S-\beta)^T-(\sigma_y^2I+2\beta\beta^T)\\
&= O(\sigma_y^2 (\log^2(\frac{|S|}{\tau})\frac{\tau}{|S|}+\epsilon)) \;.
\end{align}
Given $|S|=N=\Omega( (d/\eps^2) \poly\log(d/\eps\tau))$, 
we have $\log^2(\frac{|S|}{\tau})\frac{\tau}{|S|}=O(\epsilon)$, 
and hence with probability at least $1-\tau/5$, $\|M_S-(\sigma_y^2 I + \beta\beta^T)\|=O(\epsilon \sigma_y^2)$.

\end{proof}
This completes the proof of Proposition~\ref{prop:good-set}.

\subsection{Handling Approximate Identity Covariance}\label{sec:approxcov}
In this subsection, we discuss the case where the covariance matrix of $X$ is not exactly identity, 
but instead satisfies $(1-\eps) I \preceq \Sigma \preceq (1+\eps)I.$ We will prove that a slightly 
modified deterministic regularity condition (i.e., Definition~\ref{def:good-set}) still holds for this setting. 
The same algorithm yields an estimate of $\beta$ with the same guarantee as the exact case.

\begin{proposition} \label{prop:approx-good-set}
Let $G \sim \normal(0, \Sigma)$ where $(1-\eps)I\prec \Sigma\prec (1+\eps)I$ and $\eps, \tau>0$.
If the multiset $S$ consists of {$N = \Omega( (d/\eps^2) \poly\log(d/\eps\tau))$} labeled samples 
$(X, y) \in \R^d \times \R$, where $X \sim G$ and $y=\beta^T X +\eta$, {where $\eta \sim \normal(0,\sigma^2)$,}
then $S$ is $(\eps,\tau)$-good with respect to $(G, \beta)$ with probability at least $1-\tau$.
\end{proposition}
\begin{proof}
We prove the all the conditions in Definition~\ref{def:good-set} will be satisfied 
by reducing to the identity covariance case and applying 
Proposition~\ref{prop:approx-good-set}. Let $Z=\Sigma^{-1/2}X$, we have that pair $(y,Z)$ follows the setting of Proposition~\ref{prop:good-set}, where $Z\sim \normal(0,I)$ and $y = \beta^T\Sigma^{1/2}Z+\sigma E$. 
Let $S_a$ be the set that contains pairs $(y,\Sigma^{-1/2}X)$. 
\begin{enumerate}
\item For condition (i), we can apply Proposition~\ref{prop:approx-good-set} and get $|\frac{yX}{\sigma_y}|\le (1+\eps) |\frac{y(\Sigma^{1/2}Z)} {\sqrt{\beta^T\Sigma\beta+\sigma^2}}|\le (1+\eps) 4\sqrt{d}\log(|S|/\tau)$ and $|\frac{y}{\sigma_y}|\le\sqrt{1+\eps}|\frac{y}{\sqrt{\beta^T\Sigma\beta+\sigma^2}}|\le \sqrt{1+\eps}\sqrt{4\log(|S|/\tau)}.$ 
\item For condition (ii), we have 
\begin{align*}
&\Pr_{X,y\sim S} [\frac{y(v\cdot X)}{\sigma_y}>T]=\Pr_{Z,y\in S}[\frac{y(v^T\Sigma^{1/2}Z)}{\sqrt{\|\beta\|^2+\sigma^2}}>T] \le \Pr_{Z,y\in S}[\frac{y\frac{v^T\Sigma^{1/2}}{\|v^T\Sigma^{1/2}\|}Z}{\sqrt{\beta^T\Sigma\beta+\sigma^2}}<\frac{T}{(1+\eps)}]\\
&\le 16\exp(-T/16(1+\eps))+\frac{\epsilon (1+\eps)}{T^2\log (N/\tau)}.
\end{align*}
\item For condition (iii),  we have $\|\beta_S-\beta\| \le  \|\E_{Z,y\in S}[y\Sigma^{1/2}Z ]-\Sigma^{1/2}\beta\|+\|\Sigma^{1/2}\beta-\beta\|\le \eps\sqrt{\beta^T\Sigma\beta+\sigma^2}+\eps\|\beta\|\le O(\eps \sigma_y)$. 
\item For condition (iv), notice that by Proposition~\ref{prop:good-set}, $\|\E_{Z,y\in S_a} [(yZ-\Sigma^{1/2}\beta)(yZ-\Sigma^{1/2}\beta)^T]-(\beta^T\Sigma\beta+\sigma^2)I-\Sigma^{1/2}\beta\beta^T\Sigma^{1/2}\|\le O(\eps (\beta^T\Sigma\beta+\sigma^2))$. Multipling $\Sigma^{1/2}$ on both sides yields $\|\E_{Z,y\in S_a} [(y\Sigma^{1/2}Z-\Sigma\beta)(y\Sigma^{1/2}Z-\Sigma\beta)^T]-(\beta^T\Sigma\beta+\sigma^2)\Sigma-\Sigma\beta\beta^T\Sigma\|\le O(\eps (\beta^T\Sigma\beta+\sigma^2))$.

 \begin{align*}
 &\|M_S-\sigma_y^2I-\beta\beta^T\| \\
 &= \|\E_{Z,y\in S_a}[(y\Sigma^{1/2}Z-\beta)(y\Sigma^{1/2}Z-\beta)^T]-\sigma_y^2I-\beta\beta^T\| \\
 &= \|\E_{Z,y\in S_a}[(y\Sigma^{1/2}Z-\Sigma\beta)(y\Sigma^{1/2}Z-\Sigma\beta)^T]\\
 &-(\Sigma\beta-\beta_S)(\Sigma\beta-\beta_S)^T + (\beta-\beta_S)(\beta-\beta_S)^T - \sigma_y^2I-\beta\beta^T\|\\
 &\le O(\eps\sigma_y^2) + \|\Sigma\beta-\beta_S\|^2+ \|\beta-\beta_S\|^2 + \|(\beta^T\Sigma\beta+\sigma^2)\Sigma-\sigma_y^2I\| + \|\Sigma\beta\beta^T\Sigma-\beta\beta^T\|\\
 &\le O(\eps\sigma_y^2) \;.
 \end{align*}
\end{enumerate}
\end{proof}

Given the above proposition, we can first robustly estimate the covariance matrix to appropriate accuracy,
using the algorithm of~\cite{DKKLMS16}, and then apply our robust LR algorithm for the known covariance case.
Since the first step can be achieved efficiently with $\tilde O(d^2/\eps^2)$ samples, we obtain the desired result.

\section{Proof of Proposition~\ref{prop:filter-lr-gaussian-known-cov}} \label{ssec:alg}

We start by proving concentration bounds for $L$.
In particular, we show:

\begin{lemma} \label{lem:spectral-bound-L}
We have that $\|M_L\|_2/\sigma_y^2 = {O\left(\log^2(|S|/|L|) + \eps |S| / |L| \right)}.$
\end{lemma}
\begin{proof}
Since $L \subseteq S$, for any $z \in \R^d$, we have that
\begin{equation} \label{eqn:L-subset-S}
|S| \cdot \Pr_{(X,y) \sim S}[yX = z] \geq |L| \cdot \Pr_{(X,y) \sim L}[yX= z] \;.
\end{equation}
Since $M_L$ is a symmetric matrix, we have $\|M_L\|_2 = \max_{\|v\|_2=1} |v^T M_L v|.$ 
So, to bound $\|M_L\|_2$ it suffices to bound $|v^T M_L v|$ for unit vectors $v.$
By definition of $M_L,$
for any $v \in \R^d$ we have that
$$|v^T M_L v| = \E_{(X,y) \sim L}[|v\cdot (yX-\beta)|^2] \;.$$
For unit vectors $v$, the RHS is bounded from above as follows:
\begin{align*}
\E_{(X,y) \sim L}\left[\frac{|v\cdot (yX-\beta)|^2}{\sigma_y^2} \right] 
& =  2 \int_{0}^{\infty} \Pr_{(X,y) \sim L} \left[\left|\frac{v\cdot(yX-\beta)}{\sigma_y}\right|>T\right] T dT\\
& \le 2 \int_{0}^{\infty} \Pr_{(X,y) \sim L}\left[\left|\frac{v\cdot (yX)}{\sigma_y}\right|>T-\frac{|v\cdot \beta|}{\sigma_y}\right] T dT \\
& \le 2 \int_{0}^{\infty} \Pr_{(X,y) \sim L}\left[\left|\frac{v\cdot(yX)}{\sigma_y}\right|> T-1\right] T dT\\
& \le 2 \int_0^{4\sqrt{d} \log(N/\tau)} \min \left\{ 1, \frac{|S|}{|L|} \cdot \Pr_{(X,y) \sim L}\left[\left|\frac{v\cdot(yX)}{\sigma_y}\right|> T-1\right]   \right\} TdT\\
& \ll \int_0^{4\log(|S|/|L|)} T dT \\
& + (|S|/|L|) \int_{4\log(|S|/|L|)}^{4\sqrt{d} \log(N/\tau)}  \Big(\exp(-T/16)+\frac{\eps}{\new{T^2\log(N/\tau)}} \Big)  T dT\\
& \ll  \log^2(|S|/|L|) + \eps \cdot |S|/|L| \;,
\end{align*}
where the third line follows from the fact that $|v\cdot \beta|/\sigma_y \le 1$, the fourth line holds since the probability must be less than $1$, $S$ satisfies condition (i) of Definition~\ref{def:good-set} and Equation~\ref{eqn:L-subset-S}; and the fifth line follows from condition (iii) of Definition~\ref{def:good-set}.
\end{proof}

\begin{corollary}\label{cor:MSPrimeBound}
We have that 
$M_{S'} =  (\sigma_y^2I+\beta \beta^T) + (|E|/|S'|)M_E + O(\sigma_y^2 \eps\log^2(1/\eps))$,
where the $O(\sigma_y^2\eps\log^2(1/\eps))$ term denotes a matrix of spectral norm $O(\sigma_y^2 \eps\log^2(1/\eps))$.
\end{corollary}
\begin{proof}
 By definition, we have that $|S'| M_{S'} = |S|M_S - |L|M_L + |E|M_E$. Thus, we can write 
 \begin{align}
 M_{S'} &= (|S|/|S'|)M_S-(|L|/|S'|)M_L+(|E|/|S'|)M_E\\
&= \sigma_y^2 I + \beta \beta^T + O(\eps \sigma_y^2) + O(\sigma_y^2 \eps \log^2(1/\eps)) + (|E|/|S'|)M_E \;,
\end{align}
where the second line uses the fact that $1 - 2\eps \le |S|/|S'| \le1 + 2\eps$, the goodness of $S$ (condition (iv) in Definition~\ref{def:good-set}), and Lemma~\ref{lem:spectral-bound-L}. Specifically, Lemma~\ref{lem:spectral-bound-L} implies that $(|L|/|S'|)\|M_L\|_2 = O(\sigma_y^2 \eps \log^2(1/\eps))$. 
Therefore, we have that
$M_{S'} =\sigma_y^2 I + \beta \beta^T + O(\sigma_y^2 \eps \log^2(1/\eps)) + (|E|/|S'|)M_E$.
\end{proof}

\begin{lemma}\label{lem:meansBoundLem}
We have that ${\beta_{S'}-\beta} = (|E|/|S'|)({\beta_E-\beta}) + O(\sigma_y \eps{\log(1/\eps)})$,
where the $O(\eps{\sigma_y\log(1/\eps)})$ term denotes a vector with $\ell_2$-norm at most $O(\sigma_y \eps{\log(1/\eps)})$.
\end{lemma}
\begin{proof}
By definition, we have that $|S'|(\beta_{S'}- \beta) = |S|({\beta_S-\beta}) + |E|({\beta_E-\beta}) - |L|(\beta_L-\beta)$. Since $S$ is a good set, by condition (iii) of Definition~\ref{def:good-set}, we have $\|\beta_S-\beta\| \le O(\sigma_y\eps)$. Since $1-2\eps \le |S|/|S'| \le 1 + 2\eps$, it follows that $(|S|/|S'|)\|\beta_S-\beta\|=O(\sigma_y\eps)$. Using the valid inequality $\|M_L\|_2 \ge \|\beta_L-\beta\|_2^2$ and Lemma~\ref{lem:spectral-bound-L}, we obtain that $\|\beta_L - \beta\|_2/\sigma_y \le O(\log(|S|/|L|) + \sqrt{\eps |S|/|L|})$. Therefore, $(|L|/|S'|)\|\beta_L-\beta\|_2/\sigma_y \le O((|L|/|S|) \log(|S|/|L|) + \sqrt{\eps |L|/|S|})\le O(\eps \log(1/\eps))$.
In summary, we have ${\beta_{S'}-\beta} = (|E|/|S'|)({\beta_E-\beta}) + O(\sigma_y \eps{\log(1/\eps)})$ as desired. This completes the proof of the lemma.
\end{proof}

\begin{corollary}\label{cor:MApproxCor-G}
We have 
$\widehat{M_{S'}}- ({\sigma'_y}^2I+{\beta_{S'} \beta_{S'}^T})  = (|E|/|S'|)M_E + A+B$, where $\|A\|_2\le C_1\sigma_y^2\eps\log^2(1/\eps)$, $\|B\|_2 \le C_2\sigma_y(|E|/|S'|)\sqrt{\|M_E\|_2}$.
\end{corollary}
\begin{proof}
By definition, we can write
\begin{align*}
\widehat{M_{S'}}- ({\sigma'_y}^2I+{\beta_{S'} \beta_{S'}^T})  = &M_{S'}- ({\sigma'_y}^2I+{\beta_{S'} \beta_{S'}^T}) -(\beta_{S'}-\beta)(\beta_{S'}-\beta)^T\\
%=&  {\beta \beta^T-\beta_{S'}\beta_{S'}^T}+  (|E|/|S'|)M_E + O(\sigma_y^2 \eps\log^2(1/\eps)) - (\beta_{S'}-\beta)(\beta_{S'}-\beta)^T\\ 
=& 2(\beta-\beta_{S'})\beta_{S'}^T +(|E|/|S'|)M_E + O(\sigma_y^2 \eps\log^2(1/\eps))\\
=& O(\sigma_y \|\beta_{S'}-\beta \|_2)+ (|E|/|S'|)M_E + O(\sigma_y^2 \eps\log^2(1/\eps))\\
=& (|E|/|S'|)M_E+O(\sigma_y(|E|/|S'|)\sqrt{\|M_E\|_2})  + O(\sigma_y^2\eps\log^2(1/\eps)) \;,
\end{align*}
where the first line follows by definition, second line follows from Corollary~\ref{cor:MSPrimeBound}, 
the third line follows from the fact that $\|\beta_{S'}\|\le \sigma_y' \le O(\sigma_y)$, 
the fourth line follows from Lemma~\ref{lem:meansBoundLem}. 
This completes the proof.
\end{proof}

\noindent {\bf Case of Small Spectral Norm.} 
We are now ready to analyze the case that the 
vector $\beta^{S'}$ is returned by the algorithm in Step \ref{step:bal-small-G}.
In this case, we have that 
$\lambda^{\ast} \eqdef \|\widehat{M_{S'}}-({\sigma'_y}^2I+\beta_{S'} \beta_{S'}^T)\|_2 = O({\sigma'_y}^2\eps\log^2(1/\eps))=O({\sigma_y}^2\eps\log^2(1/\eps))$, 
which implies for any unit vector $v$, we have that
$v^T\left(\widehat{M_{S'}}-({\sigma'_y}^2I+\beta_{S'} \beta_{S'}^T)\right)v = O(\sigma_y^2\eps\log^2(1/\eps))$. 
Since $\|M_E\|_2 \geq \|\mu^E-\mu^G\|_2^2$, Lemma \ref{lem:meansBoundLem} gives that
$$
\frac{\|\beta_{S'}-\beta\|_2}{\sigma_y} \leq (|E|/|S'|) \sqrt{\frac{\|M_E\|_2}{\sigma_y^2}} + O(\eps{\log(1/\eps)})  \;.
$$
If $\|M_E\|=O(\sigma_y^2)$, the proof will be complete. Otherwise Corollary~\ref{cor:MApproxCor-G} becomes $\widehat{M_{S'}}- ({\sigma'_y}^2I+{\beta_{S'} \beta_{S'}^T})  = O((|E|/|S'|)M_E+\sigma_y^2\eps\log^2(1/\eps)) $ which implies
$$
(|E|/|S'|) \|M_E\|_2 = O(\sigma_y^2\eps\log^2(1/\eps)) \;,
$$
and hence $\frac{\|\beta_{S'}-\beta\|_2}{\sigma_y} = O(\eps\log(1/\eps))$. 
{This proves part (i) of Proposition~\ref{prop:filter-lr-gaussian-known-cov}.}

\medskip

\noindent {\bf Case of Large Spectral Norm.} 
We next show the correctness of the algorithm 
when it returns a filter in Step \ref{step:bal-large-G}.

We start by proving that if $\lambda^{\ast} \eqdef \|\widehat{M_{S'}}-({\sigma'_y}^2I+\beta_{S'} \beta_{S'}^T)\|_2 > C{\sigma'_y}^2\eps\log^2(1/\eps)$, 
for a sufficiently large universal constant $C$,
then a value $T$ satisfying the condition in Step \ref{step:bal-large-G} exists. 

We need to prove the following two facts of about $M_E$: 
(1) $\frac{|E|}{|S'|}\|M_E\|= O(\lambda)$, and (2) $(|E|/|S'|)(v^*)^T M_E v^*   = \Omega (\lambda^{\ast})$. 

If $\|M_E\|<64C_2^2\sigma_y^2$, we must have $\lambda<72\eps C_2^2 \sigma_y^2+C_1\sigma_y^2\eps \log^2(1/\eps)<C{\sigma'_y}^2\eps\log^2(1/\eps)$, 
which yields a contradiction. If $\|M_E\|\ge 64C_2^2\sigma_y^2$,  we have $\|B\|_2\le \frac{|E|}{8|S'|}M_E$ in Corollary~\ref{cor:MApproxCor-G}, 
which will be assumed for the rest of the proof. For sufficiently large $C$, we have 
\begin{align*}
\frac{7}{4}\lambda &\ge \lambda+C_1\sigma_y^2 \eps\log^2(1/\eps)\ge\|\frac{|E|}{|S'|}M_E+B\|
\ge \|\frac{|E|}{|S'|}M_E\|-\|B\|\ge \frac{7}{8}\|\frac{|E|}{|S'|}M_E\| \;,
\end{align*} 
where $B$ is the matrix defined in Corollary~\ref{cor:MApproxCor-G}. 
Hence, we have $\frac{|E|}{|S'|}\|M_E\|\le 2\lambda$. 
The proof of the first fact is complete.  

By the definition of $\lambda^*$: 
\begin{equation}
\lambda^* = {v^*}^T\left(\widehat{M_{S'}}-({\sigma'_y}^2I+{\beta' \beta'^T})\right)v^*  = (|E|/|S'|){v^*}^TM_Ev^* + {v^*}^TAv^*+{v^*}^TBv^* \;,
\end{equation}
which is equivalent to  
\begin{equation}\label{eqn:large-me1}
(|E|/|S'|){v^*}^TM_Ev^*= \lambda^* - \|A\|_2-\|B\|_2 \ge c \lambda \;,
\end{equation}
for a non-negative constant $c$. 
The proof of the second fact is complete.

Moreover, using the inequality $\|M_E\|_2 \geq \|\beta_E-\beta\|_2^2$
and Lemma \ref{lem:meansBoundLem} as above, we get that
\begin{equation} \label{eqn:delta1}
\frac{\|\beta_{S'}-\beta\|_2}{\sigma'_y} \leq (|E|/|S'|) \frac{\sqrt{{\|M_E\|_2}}}{\sigma'_y} + O( \eps{\log(1/\eps)}) \leq \delta/2 \;,
\end{equation}
where we used the fact that $\delta \eqdef 3\sqrt{\eps \lambda^{\ast}}/\sigma'_y > C' \eps {\log(1/\eps)}.$

Suppose for the sake of contradiction that for all $T>0$ we have that
\begin{equation*} 
\Pr_{X\sim S'}\left[\frac{|v^{\ast} \cdot (yX-\beta_{S'})|}{\sigma'_y}>T+\delta \right] 
\leq 32\exp(-T/16)+8 \frac{\eps}{T^2\log(N/\tau)} \;,
\end{equation*}
which implies
\begin{equation*}
\Pr_{X\sim S'}\left[\frac{|v^{\ast} \cdot (yX-\beta')|}{\sigma_y}>(1+O(\epsilon\log \frac{1}{\epsilon}))(T+\delta) \right] 
\leq 32\exp(-T/16)+8 \frac{\eps}{T^2\log(N/\tau)} \;.
\end{equation*}
Using (\ref{eqn:delta1}), and assume that $O(\eps\log(1/\eps))\le 1$, we obtain that for all $T>0$ we have that
\begin{equation} \label{eqn:contradiction}
\Pr_{X\sim S'}\left[\frac{|v^{\ast} \cdot (Xy-\beta)|}{\sigma_y}>T+3\delta \right] \leq 32\exp(-T/32)+32 \frac{\eps}{{T^2\log(N/\tau)}} \;.
\end{equation}
\iffalse
Since $E \subseteq S',$ for all $x \in \R^d$ we have that
$|S'|\Pr_{X\sim S'}[X=x] \geq |E| \Pr_{Y\sim E}[Y=x].$
This fact combined with (\ref{eqn:contradiction}) implies that for all $T>0$
\begin{equation} \label{eqn:contradiction2}
\Pr_{X\sim E}\left[\frac{|v^{\ast} \cdot (Xy-\beta)|}{\sqrt{\|\beta\|^2+\sigma^2}}>T+\delta/2 \right]  \ll (|S'|/|E|)\left(\exp(-T/2)+ \frac{\eps}{\new{T^2\log\left(d \log(\frac{d}{\eps\tau})\right)}} \right) \;.
\end{equation}
\fi
We now have the following sequence of
inequalities:
\begin{align*}
(v^*)^T \frac{M_E}{\sigma_y^2} v^* &=  \E_{X,y\sim E}\left[\frac{|v^{\ast} \cdot (Xy-\beta)|^2}{\sigma_y^2}\right]  = 2 \int_{0}^{\infty} \Pr_{X,y\sim E}\left[\frac{|v^{\ast} \cdot (Xy-\beta)|}{\sigma_y}>T\right] T dT\\
& \leq 2 \int_0^{\infty} \min \left\{ 1,  \frac{|S'|}{|E|}  \Pr_{X \sim S'}\left[\frac{|v^{\ast} \cdot (Xy-\beta)|}{\sigma_y}>T\right]  \right\} TdT\\
& \ll \int_0^{32\log(|S'|/|E|)+3\delta} T dT + (|S'|/|E|) \int_{32\log(|S'|/|E|)+3\delta}^{\infty}   \Pr_{X \sim S'}\left[\frac{|v^{\ast} \cdot (Xy-\beta)|}{\sigma_y}>T\right] T dT\\
& \ll \int_0^{32\log(|S'|/|E|)+3\delta} T dT + (|S'|/|E|) \int_{32\log(|S'|/|E|)+3\delta}^{\infty}  
\Big( \exp(-(T-3\delta)/32)+\frac{\eps}{(T-3\delta)^2\log\left(N/\tau\right)} \Big)  T dT\\
& \ll \int_0^{32\log(|S'|/|E|)+3\delta} T dT + (|S'|/|E|) \int_{32\log(|S'|/|E|)}^{\infty}  
\Big( \exp(-T/32)+\frac{\eps}{T^2\log\left(N/\tau\right)} \Big) (T+3\delta) dT\\
& \ll  \log^2(|S'|/|E|) + \delta^2 + 1 + \log(|S'|/|E|) + \eps \cdot |S'|/|E| + \delta + \delta\eps\frac{|S'|/|E|}{ \log(|S'|/|E|)\log(N/\tau)}\\
& \ll  \log^2(|S'|/|E|) + \frac{\eps \lambda^{\ast}}{\sigma_y^2} + \eps \cdot |S'|/|E| + \frac{|S'|/|E|\eps^{3/2}\sqrt{\lambda^*}}{\sigma_y\log(|S'|/|E|)\log(N/\tau)}\;.
\end{align*}
Rearranging the above, we get that
\begin{align*}
(|E|/|S'|) (v^*)^T \frac{M_E}{\sigma_y^2} v^*  \ll (|E|/|S'|)\log^2(|S'|/|E|) + (|E|/|S'|) \eps \lambda^{\ast} + \eps +\frac{\eps^{3/2}\sqrt{\lambda^*}}{\log(1/\eps)\log(N/\tau)}\\
\ll O(\eps \log^2(1/\eps) + \frac{\eps^2 \lambda^{\ast}}{\sigma_y^2}+\frac{\eps^{3/2}\sqrt{\lambda^*}}{\sigma_y\log(1/\eps)\log(N/\tau)}).
\end{align*}
Combined with (\ref{eqn:large-me1}), we obtain the following: there exists constants $c,C_1,C_2,C_3$ such that 
$$
(c-\eps^2C_1)\lambda^* - C_2\frac{\sigma_y\eps^{3/2}\sqrt{\lambda^*}}{\log(1/\eps)\log(N/\tau)} \le C_3\sigma_y^2\eps \log^2(1/\eps).
$$
Notice that because $\lambda^*\ge \sigma_y^2\eps$, we have $\sigma_y\sqrt{\eps\lambda^*}\le \lambda$. By the above equation, for sufficiently small $\eps$, we have 
$\lambda^{\ast} = O(\sigma_y^2\eps \log^2(1/\eps))$, which is a contradiction if $C$ is sufficiently large. Therefore, it must be the case that for some value of $T$ the condition in Step \ref{step:bal-large-G} is satisfied.

The following claim completes the proof:
\begin{claim} \label{claim:filter}
We have that
$
\Delta(S,S'') < \Delta(S,S') \;.
$
\end{claim}
\begin{proof}
Recall that $S' = (S\setminus L) \cup E,$ with $E$ and $L$ disjoint multisets such that $L \subset S.$
We can similarly write $S''=(S \setminus L') \cup E',$ with $L'\supseteq L$ and $E'\subset E.$
Since $$\Delta(S,S') - \Delta(S,S'')  = \frac{|E \setminus E'| - |L' \setminus L| }{|S|},$$
it suffices to show that $|E \setminus E'| > |L' \setminus L|.$
Note that $|L' \setminus L|$ is the number of points rejected by the filter that lie in $S \cap S'.$
%By Claim~\ref{cernCor-G} and Claim~\ref{claim:mean-l2-delta-G}, it follows that 
Note that the fraction of elements of $S$
that are removed to produce $S''$ (i.e., satisfy $|\frac{|v^{\ast} \cdot (yX-\beta^{S'})|}{\sigma'_y}|>T+\delta$) is at most $16\exp(-T/16) + \frac{\eps}{T^2 \log(N/\tau)}$.
This follows from property (ii) of Definition~\ref{def:good-set}.

Hence, it holds that $|L' \setminus L| \leq (16\exp(-T/16) + \frac{\eps}{T^2 \log(N/\tau)}) |S|.$
On the other hand, Step~\ref{step:bal-large-G} of the algorithm ensures that the fraction of elements of $S'$ that are rejected
by the filter is at least $32\exp(-T^2/16)+8\eps/\new{\alpha})$. Note that
$|E \setminus E'|$ is the number of points rejected by the filter that lie in $S' \setminus S.$
Therefore, we can write:
\begin{align*}
|E\setminus E'| & \geq (32\exp(-T/16)+8\frac{\eps}{T^2 \log(N/\tau)})|S'| - (16\exp(-T/16) + \frac{\eps}{T^2 \log(N/\tau)})  |S| \\
				& \geq (32\exp(-T/16)+8\frac{\eps}{T^2 \log(N/\tau)})|S|/2 - (16\exp(-T/16) + \frac{\eps}{T^2 \log(N/\tau)}) |S| \\
				& \geq  (16\exp(-T/16) + 7\frac{\eps}{T^2 \log(N/\tau)}) |S| \\
				& > |L' \setminus L| \;,
\end{align*}
where the second line uses the fact that $|S'| \ge |S|/2$
and the last line uses the fact that $|L' \setminus L| / |S| \leq 16\exp(-T/16) + \frac{\eps}{T^2 \log(N/\tau)}.$
This completes the proof of the claim.
\end{proof}

\iffalse
\new{Where do we need these things?}

\begin{fact}
For every affine function $L:\R^d \to \R$ \new{such that $L(x) = v \cdot x-T$, $\|v\|_2=1$,}
we have that 
$\left|\Pr_{X \sim S}[L(X) \ge 0] - \Pr_{X\sim G}[L(X) \ge 0] \right| \leq \frac{\eps}{\new{T^2\log\left(d \log(\frac{d}{\eps\tau})\right)}} \;.$
\end{fact}
\begin{fact}
$\Pr[X>(1+\delta)\mu] \le \exp(-\frac{\delta^2\mu}{3})$ when $0<\delta<1$. $\Pr[X>(1+\delta)\mu] \le \exp(-\frac{\delta\log(1+\delta)\mu}{2})$ when $0<\delta$. 
\end{fact}
\fi

\section{Deterministic Regularity Conditions for Algorithm~\ref{alg:filter-LR-identity-2}}\label{sec:good-set-2}

We start by formally defining the set of deterministic regularity conditions under which our main algorithm succeeds:

\begin{definition} \label{def:good-set-2}
Let $G \sim \normal(0, I_d)$, $\beta \in \R^d$, and $\eps,\tau >0$.
We say that a multiset $S$ of elements in $\R^d \times \R$ is {\em $(\eps,\tau)$-representative} 
\new{(with respect to $(G, \beta)$)}
if the following conditions are satisfied:
\begin{enumerate}
%\item
%\begin{itemize}
%\item[(i)] For all $(X, y) \in S$,
%we have $\|(y-\beta \cdot X) X\|_2 \leq 4\sqrt{d} \log (|S|/\tau) \sigma$ and $(y-\beta \cdot X)/\sigma \le \sqrt{4\log(|S|/\tau)}$.
%\end{itemize}

%Good set stuff for Gaussian filter for beta.X
\item \label{cond:good-for-X-filter}
\begin{itemize}
\item[(i)] For all $(X,y) \in S$, $\|X\|_2 \leq O(\sqrt{d \log(|S|/\tau))}$.
\item[(ii)] For any unit vector $v$ and $T> 0$, we have
$$\Pr_{(X,y) \sim S}[|v \cdot X| > T] \leq 5\exp(-T^2/4)
 + \frac{\eps^2}{T^2 \log(d \log(d/\eps \tau))} \;.$$
 \item[(iii)]  $\|\E_S[X]\|_2 \leq O(\eps).$
 \item[(iv)] $\|\E_S[X X^T - I]\|_2 \leq O(\eps) .$
\end{itemize}

%good set stuff for filter of (y-X.\beta')
\item \label{cond:good-for-one-dim-filter}
For all $\beta' \in \R^d$, let $\sigma_{\beta'}^2=\sigma^2+\|\beta-\beta'\|_2^2$. 
Then the following hold:
\begin{itemize}
\item[(i)] For all $(X,y) \in S$, $|y-\beta'\cdot X| \leq O(\sqrt{d \log(|S|/\tau) \sigma_{\beta'}})$.
\item[(ii)] For any $T> 0$,
$$\Pr_{(X,y) \sim S}[|y - \beta' \cdot X| > T] \leq 5\exp(-T^2/4\sigma_{\beta'}^2)
 + \frac{\eps^2 \sigma_{\beta'}^2}{T^2 \log(d \log(d/\eps \tau))} \;.$$
 \item[(iii)]  $|\E_S[y - \beta' \cdot X]| \leq O(\eps \sigma_{\beta'})$.
 \item[(iv)] $|\E_S[(y - \beta' \cdot X)^2] - \sigma_{\beta}^2| \leq O(\eps \sigma_{\beta'}^2)$.
 \end{itemize}

% Good set stuff for call to main algorithm
\item \label{cond:good-for-lin-reg-filter}
For all $\beta' \in \R^d$, let $\sigma_{\beta'}^2=\sigma^2+\|\beta-\beta'\|_2^2$ and for $5\sqrt{\ln(1/\eps)} \leq T' \leq 7\sqrt{\ln(1/\eps)}$ 
let $R_{\beta',T'}=\{(X,y) \in S : |y-\beta' \cdot X| \leq T' \sigma_{\beta'} \}$. Then we have
\begin{itemize}
\item[(i)] For all $(X, y) \in R_{\beta',T'}$,
$\|(y-\beta' \cdot X) X\|_2 \leq O(\sqrt{d/\eps} \log (|S|/\tau) \sigma_{\beta'})$.

\item[(ii)]  For every $v\in \R^d$ with $\|v\|_2=1$,
we have that 
$$\Pr_{(X, y) \sim R_{\beta',T'}}\left[ \frac{(y-\beta' \cdot X)(v\cdot X)}{\sigma_{\beta'}} >  T \right] \le 24\exp(-T/16)+\frac{\epsilon}{T^2\log (N/\tau)} \;.$$

\item[(iii)]  We have that $\left\|{\E_{R_{\beta',T'}}[(y-\beta' \cdot X)X] - (\beta-\beta')} \right\|_2\leq O(\eps \log(1/\eps) \sigma_{\beta'}) .$

\item[(iv)]  We have that $\left\|M_{S,\beta'} - (\sigma_{\beta'}^2I+{(\beta-\beta') (\beta-\beta')^T}) \right\|_2 \leq O( \log^2(1/\eps) \eps \sigma_{\beta}^2)$, where $M_{S,\beta'}=\E_{R_{\beta',T'}}[((y -\beta' \cdot X) X - (\beta-\beta')) ((y-\beta' \cdot X)X - (\beta-\beta'))^T]$
\end{itemize}
\item $|\Pr_{(X, y) \sim S}[v\cdot X-y > T] - \Pr_{D}[v\cdot X-y ]| \leq \eps/10$.
\end{enumerate}
\end{definition}

We now establish that a sufficiently large set of uncorrupted samples 
will satisfy the above conditions with high probability:

\begin{proposition}\label{prop:good-set-2} If $S$ is a set of uncorrupted samples with size $|S|$ larger than $O(d\polylog(d/\eps\tau)/\eps^2)$, then except with probability $1/\tau$, $S$ is $(\eps,\tau)$-representative.\label{prop:good-set-nobeta}
\end{proposition}
The rest of this section will focus on establishing the above proposition. Condition 1(i) and (ii) follow from Claim~\ref{claim:good-set-tail} and (iii) and (iv) follow from Lemma~\ref{lem:Gaussian-basics-a}. Condition 2 (i) and (ii) follow from Lemma  \ref{lem:y'-thing} , (iii) and (iv) from Lemma \ref{lem:Gaussian-basics-b}. Condition 3 (i) follows from Claim \ref{claim:good-set-tail} and the bound on $|y-\beta' \cdot X|$ in $R_{\beta',T'}$, 
(ii) follows from Lemma \ref{lem:combined}, (iii) and (iv) are given by Corollary \ref{cor:hard-bit}.

It follows from standard results on estimating the mean and covariance matrix of a Gaussian that:
\begin{lemma} \label{lem:Gaussian-basics-a}
For $N=\Omega(\frac{d}{\epsilon^2} \poly\log(d/\eps\tau))$, with probability at least $1-\tau/10$,
we have that $\|\E_{(X, y) \sim S}[ X]\|_2 \leq O(\eps)$ and $\|\E_{(X, y) \sim S}[X X^T] -I \|_2 \leq O(\eps)$.\end{lemma}

In addition to Claim \ref{claim:good-set-tail}, we want that:
\begin{lemma} \label{lem:y'-thing}
If Claim \ref{claim:good-set-tail} holds then,
with probability at least $1-\tau/10$, we have that, for any $\beta'$ and any $T>0$,
$\Pr_{(X, y) \sim S}\left[|y - \beta' \cdot X| > T\right] \leq 10 \exp(-T^2/16\sigma_{\beta'}^2)+\frac{4\eps^2\sigma_{\beta'}^2}{T^2\log^3\left(|S|/\tau\right)}.$ where $\sigma_{\beta'}^2= \|\beta'-\beta\|_2^2+\sigma^2$.
\end{lemma}
\begin{proof}
By the triangle inequality, $|y - \beta' \cdot X| \leq |(\beta-\beta') \cdot X| + |y - \beta \cdot X|.$
By Claim  \ref{claim:good-set-tail}, we have that for all $T > 0$ and $\beta'$ that
$$\Pr_{(X, y) \sim S}\left[|(\beta-\beta') \cdot X| > T/2 \right] \leq 5 \exp(-T^2/16\|\beta-\beta'\|_2^2)+\frac{4\eps^2\|\beta-\beta'\|_2^2}{T^2\log^3\left(|S|/\tau\right)}.$$
Since $\frac{1}{\sigma}(y-\beta \cdot X) \sim \normal(0,1)$, we can apply Claim \ref{clm:single-v} to it to get that, except with probability $1-\exp(|S|\log^3\left(|S|/\tau\right)/\eps^2) \geq 1 -\tau/10$ that, for all $T$,
$$\Pr_{(X, y) \sim S}\left[|(y-\beta \cdot X) \cdot X| > T/2 \right] \leq 5 \exp(-T^2/16\sigma^2)+\frac{4\eps^2\sigma^2}{T^2\log^3\left(|S|/\tau\right)}.$$
Since $\sigma_{\beta'}=\|\beta'-\beta\|_2^2+\sigma^2$, we have:
\begin{align*}
\Pr_{(X, y) \sim S}\left[|y - \beta' \cdot X| > T\right] & \leq \Pr_{(X, y) \sim S}\left[|(\beta-\beta') \cdot X| > T/2 \right] + \Pr_{(X, y) \sim S}\left[|(y-\beta \cdot X) \cdot X| > T/2 \right] \\
& \leq 10 \exp(-T^2/16\sigma_{\beta'}^2)+\frac{4\eps^2\sigma_{\beta'}^2}{T^2\log^3\left(|S|/\tau\right)}.
\end{align*}
\end{proof}

\begin{lemma} \label{lem:Gaussian-basics-b}
For $N=\Omega(\frac{d}{\epsilon^2} \poly\log(d/\eps\tau))$, with probability at least $1-\tau/10$, for any $\beta'\in \R^d$
we have that $|\E_{(X, y) \sim S}[(y - \beta' \cdot X)]| \leq O(\eps \sigma_{\beta'})$, $|\E_{(X, y) \sim S}[(y - \beta' \cdot X)^2] -  \sigma_{\beta'}^2| \leq O(\eps \sigma_{\beta'}^2)$, where $\sigma_{\beta'}^2=\sigma^2+\|\beta-\beta'\|_2^2$
\end{lemma}
\begin{proof}
It follows from standard results on estimating the mean and covariance matrix of a Gaussian that 
$|\E_{(X, y) \sim S}[(y - \beta \cdot X)]| \leq O(\eps \sigma)$, $|\E_{(X, y) \sim S}[(y - \beta \cdot X)^2] -  \sigma^2| \leq O(\eps \sigma^2)$. By Lemma~\ref{lem:Gaussian-basics-a}, we have 
$|\E_{(X, y) \sim S}[(y - \beta' \cdot X)]|\le |\E_{(X, y) \sim S}[y - \beta \cdot X]|+ |(\beta - \beta')\E_{(X, y) \sim S}[ \cdot X]|\le \eps\sigma+\|\beta-\beta'\|_2]\le O(\eps\sigma+\|\beta'-\beta\|\eps) = O(\eps\sigma_{\beta'})$. 
By Proposition~\ref{lem:Hanson-Wright-thing} and Lemma~\ref{lem:Gaussian-basics-a},  we have 
\begin{align*}
|\E_{(X, y) \sim S}[(y - \beta' \cdot X)^2]-\sigma_{\beta'}^2| &\\
&\le |\E_{(X, y) \sim S}[(y - \beta \cdot X)^2]-\sigma^2|+2 |(\beta - \beta')\cdot \E_{(X, y) \sim S}[ (y-\beta\cdot X)X]|\\
&+|(\beta - \beta')^TE_{(X, y) \sim S}[ XX^T](\beta - \beta')-\|\beta-\beta'\|_2^2|\\
&\le O(\eps\sigma^2+\|\beta-\beta'\|_2\eps\sigma+\|\beta-\beta'\|_2^2\eps)=O(\eps(\sigma^2+\|\beta-\beta'\|^2)) \;.
\end{align*}
\end{proof}

\begin{proposition} \label{lem:Hanson-Wright-thing}
For $N=\Omega(\frac{d}{\epsilon^2} \poly\log(d/\eps\tau))$, with probability at least $1-\tau/10$,  
$\|\E_{(X, y) \sim S}[(y - \beta \cdot X) X]\|_2 \leq \eps \sigma$.
\end{proposition}
\begin{proof}
The proof is very similar to that of Lemma \ref{lem:betaS-error}.
\end{proof}

Note that $\Pr_S[|y - \beta' \cdot X| > 5\sqrt{\ln(1/\eps)} \sigma_{\beta'}] \leq \eps$. As long as $T'\ge 5\sqrt{\ln(1/\eps)}$ and $\eps<1/11$ we will have:
\begin{corollary} \label{cor:new-algo-ytail}
$\Pr_{(X, y) \sim R_{\beta',T'}}\left[|y - \beta' \cdot X| > T\right] \leq 12 \exp(-T^2/16\sigma_{\beta'}^2)+\frac{5\eps^2\sigma_{\beta'}^2}{T^2\log^3\left(|S|/\tau\right)}.$\end{corollary}

We can combine this corollary with Claim~\ref{claim:good-set-tail} in a similar way to Lemma~\ref{lem:good-set-two}  to obtain that:
\begin{lemma} \label{lem:combined}
If Claim \ref{claim:good-set-tail}  and the previous Lemma \ref{lem:y'-thing} hold, we have that, for any $\beta'$, unit vector $v$ and any $T>0$,
$\Pr_{(X, y) \sim R_{\beta',T'}}\left[|(y - \beta' \cdot X) (v \cdot X)| > T\right] \leq 24 \exp(-T/16\sigma_{\beta'}^2)+\frac{2\eps \sigma_{\beta'}^2}{{T^2\log^2\left(|S|/\tau\right)}}$, where $\sigma_{\beta'}^2= \|\beta'-\beta\|_2^2+\sigma^2$ and $R_{\beta',T'}$ is $S$ with the elements with $|y - \beta' \cdot X| \leq T' \sigma_{\beta'}$, where $5\sqrt{\ln(1/\eps)}\leq T' \leq 50\sqrt{\ln(1/\eps})$.
\end{lemma}
\begin{proof}

By  Claim~\ref{claim:good-set-tail}, $\Pr_{X\sim S}[|v \cdot X| > T] < 1/|S|$ when $T \geq \eps^2|S|=O(d \polylog(d/\eps\tau))$. 
Thus, the maximum value of $T$ for which $\Pr_{(X, y) \sim R_{\beta',T'}}\left[|(y - \beta' \cdot X) (v \cdot X)| > T\right]$ is non-zero is $T' \sigma_{\beta'} \cdot O(d \polylog(d/\eps\tau))$. When $T < 16\sigma_{\beta'}$, the lemma is trivial. So we need to show it for $16\sigma_{\beta'} < T < O(d \polylog(d/\eps\tau))T' \sigma_{\beta'}$.

We apply Fact \ref{fact:product-bound} to obtain that
\begin{align*}
& \Pr_{(X, y) \sim R_{\beta',T'}}\left[|(y - \beta' \cdot X) (v \cdot X)| > T\right] \\
 \leq &
\sum_{i=1}^{\log_2 \min\{T',T/\sigma_{\beta'}\}} \Pr[|y - \beta' \cdot X| \geq 2^i \sigma_{\beta'} \wedge |v \cdot X| \geq T/2^i \sigma_{\beta'}] +\Pr[|v \cdot X| \geq T/\sigma_{\beta'}]\\
\leq & \sum_{i=1}^{\log_2 \min\{T',T/\sigma_{\beta'}\}} \min \{12 \exp(-2^{2i-4})+\frac{5\eps^2}{2^{2i}\log^3\left(|S|/\tau\right)},
 5 \exp(-T^2 2^{-2i-2}/\sigma_{\beta'}^2)+\frac{\eps^2 2^{2i} \sigma_{\beta'}^2}{T^2\log^3\left(|S|/\tau\right)}\},
 \end{align*}
 where the last inequality holds due to Corollary~\ref{cor:new-algo-ytail} and Claim~\ref{claim:good-set-tail}.
 When $T \geq 2 T'^2 \sigma_{\beta'}$, each term has $2^{2i} \leq T/2\sigma_{\beta'}$. In this case, we have $5 \exp(-T^2 2^{-2i-2}/\sigma_{\beta'}^2) \leq 5\exp(-T /4\sigma_{\beta'}) \leq 12 \exp(-2^{2i-4})$ for the exponential term and $\frac{2^{2i} \sigma_{\beta'}^2}{T^2} \leq \sigma_{\beta'}/2T \leq \frac{5}{2^{2i}}$. Hence, the second term is always smaller. We have 
 $5 \exp(-T^2 2^{-2i-2}/\sigma_{\beta'}^2) \leq 5 \exp(-T 2^{\log_2 T'-i} /4\sigma_{\beta'}) \leq 5 \cdot 2^{i-\log_2 T'} \exp(-T/4\sigma_{\beta'})$ and so the sum of these terms is at most $10\exp(-T/4\sigma_{\beta'})$. 
 Each of the $\frac{\eps^2 2^{2i} \sigma_{\beta'}^2}{T^2\log^3\left(|S|/\tau\right)}$ terms is bounded by $\frac{\eps^2  \sigma_{\beta'}}{2T\log^3\left(|S|/\tau\right)}$ and the sum is over $\log_2 T' = O(\ln(1/\eps)) \leq \log\left(|S|/\tau\right)$ terms and so in this case we have
 $$ \Pr_{(X, y) \sim R_{\beta',T'}}\left[|(y - \beta' \cdot X) (v \cdot X)| > T\right] \leq
 10 \exp(-T /4\sigma_{\beta'})+\frac{\eps^2  \sigma_{\beta'}}{2T\log^2\left(|S|/\tau\right)} \;.$$
 We now consider the case when $T \leq 2 T'^2 \sigma_{\beta'}$. When $2^{2i} \geq 2T/\sigma_{\beta'}$, we have
 $$12 \exp(-2^{2i-4})+\frac{5\eps^2}{2^{2i}\log^3\left(|S|/\tau\right)} \leq 12 \exp(-T/8\sigma_{\beta'}) + 
 \frac{5\eps^2}{4T \log^3\left(|S|/\tau\right)} \;.$$
 When $2^{2i} < 2T/\sigma_{\beta'}$, we have
 $$5 \exp(-T^2 2^{-2i-2}/\sigma_{\beta'}^2)+\frac{\eps^2 2^{2i} \sigma_{\beta'}^2}{T^2\log^3\left(|S|/\tau\right)}
 \leq 5\exp(-T/8\sigma_{\beta'}) +\frac{2\eps^2  \sigma_{\beta'}}{T\log^3\left(|S|/\tau\right)} \;.$$
 Since there are $1+\log_2 T/\sigma_{\beta'}$ terms and $1+\log_2 x \leq 2\exp(x/16)$ for $x \geq 16$ and $1+\log_2 T/\sigma_{\beta'} \leq 1 + 2 \log_2 T' \leq \log\left(|S|/\tau\right)$, we have that 
 $$ \Pr_{(X, y) \sim R_{\beta',T'}}\left[|(y - \beta' \cdot X) (v \cdot X)| > T\right]
 \leq 24 \exp(-T/16\sigma_{\beta'}) + \frac{2\eps^2  \sigma_{\beta'}}{T\log^2\left(|S|/\tau\right)} \;.$$
This completes the proof.
\end{proof}

Next we show that, with high probability, the expectations in 3 (iii) and (iv) are close under $D$ and $S$ when we condition them similarly:

\begin{lemma} \label{lem:fourth-moments-with-d-samples} 
For any  $2\sqrt{\ln(1/\eps)} \leq T' \leq 100\sqrt{\ln(1/\eps)}$,
let $D_{\beta',T'}$ be $D$ conditioned on $|y-\beta' \cdot X| \leq T' \sigma_{\beta'}$. Then, except with probability $\tau/10$, for all $\beta' \in \R^n$ and all $2\sqrt{\ln(1/\eps)} \leq T' \leq 100\sqrt{\ln(1/\eps)}$ and all unit vectors $v$, we have that
$$|\E_{R_{\beta',T'}}[((y -\beta' \cdot X) (v \cdot X))^2]
-\E_{D_{\beta',T'}}[((y -\beta' \cdot X) (v \cdot X))^2]| \leq \eps \sigma_{\beta'}^2$$
and
$$|\E_{R_{\beta',T'}}[((y -\beta' \cdot X) (v \cdot X))]
-\E_{D_{\beta',T'}}[((y -\beta' \cdot X) (v \cdot X))]| \leq \eps \sigma_{\beta'}.$$
\end{lemma}
\begin{proof}
We let $X'=(X,(y-\beta \cdot X)/\sigma)$ and note that $X' \sim \normal(0,I)$ in $d+1$ dimensions. We note that $(y-\beta' \cdot X)/\sigma_{\beta'}=w \cdot X'$ for some unit vector $w$. We can consider $F_{w,T'}$ to be the distribution of $X'$ conditioned on $|w \cdot X'| \leq T'$ and $R_{\beta',T'}$ as consisting of at least $(1-\eps)N$ independent samples from $F_{w,T'}$.
It suffices to show that, except with probability $\tau/10$, for all unit vectors $v,w \in \R^{d+1}$ and all $|y-\beta' \cdot X| \leq T' \sigma_{\beta'}$ that
\begin{equation} \label{eq:thingie}
|\E_{R_{\beta',T'}}[((w \cdot X') (v \cdot X'))^2]
-\E_{D_{\beta',T'}}[((w \cdot X') (v \cdot X'))^2]| \leq \eps
\end{equation}
and
\begin{equation} \label{eq:thingie-2}
|\E_{R_{\beta',T'}}[((w \cdot X') (v \cdot X'))^2]
-\E_{D_{\beta',T'}}[((w \cdot X') (v \cdot X'))^2]| \leq \eps.
\end{equation}

Now consider a fixed $v,w,T'$. For any $i$,
$$\E_{D_{\beta',T'}}[(w \cdot X') (v \cdot X'))^{2i}] \leq T'^{2i} \E_{D_{\beta',T'}}[(v \cdot X')^{2i}]
\leq  T'^{2i}(1+2\eps) \E_{X' \sim \normal(0,1)}[(v \cdot X')^{2i}] \leq (2T')^{2i} i!.$$
Thus, the central moments of $((w \cdot X') (v \cdot X'))^2$ satisfy
\begin{align*}
& \E_{D_{\beta',T'}}[|((w \cdot X') (v \cdot X'))^2-\E_{D_{\beta',T'}}[((w \cdot X') (v \cdot X'))^2]|^i] \\
& \leq 2^i \left(\E_{D_{\beta',T'}}[((w \cdot X') (v \cdot X'))^2i] + \E_{D_{\beta',T'}}[((w \cdot X') (v \cdot X'))^2]^i \right)\\
& \leq (4T')^{2i} i!.
\end{align*}
These bounds are enough to use Bernstein's inequality to show that (\ref{eq:thingie}) holds except with probability at most
$\exp(-|R_{\beta',T'}|^2\eps^2/2(|R_{\beta',T'}|\eps (4T')^2 + |R_{\beta',T'}|/(4T')^2)) \leq \exp(-N \eps^2/9T'^2)=\exp(-\Omega(N\eps^2/\ln(1/\eps)))$.

Similar moment bounds hold for $(w \cdot X') (v \cdot X')$ and we can again use Bernstein's inequality to show that the probability that (\ref{eq:thingie-2}) holds is at least $1-\exp(-\Omega(N\eps^2/\ln(1/\eps)))$. Note that we can get both  (\ref{eq:thingie}) and (\ref{eq:thingie-2}) to hold with $\eps/3$ instead of $\eps$ with probability at least $1-\exp(-\Omega(N\eps^2/\ln(1/\eps)))$.

Note that $\|X'\|_2 \leq r= O(\sqrt{d \log(|S|/\tau)})$ under $D_{\beta',T'}$ except with probability $\tau/20N$. With probability at least $1-\tau/20$, we have that $\|X'\|_2 \leq r= O(\sqrt{d \log(N/\tau)})$ for all $X' \in \R_{\beta',T'}$. 

Let $\alpha$ be a sufficiently small multiple of $\eps/r^4\sqrt{\ln(1/\eps)}$.
Let $C_w$,$C_v$ be $\alpha$-covers of the unit sphere in $\R^{d+1}$ and $C_{T'}$ be the set of multiples of  $\alpha$ between $2\sqrt{\ln(1/\eps)}$ and $100 \sqrt{\ln(1/\eps)}$. $|C_w||C_v||C_{T'}| \leq O(1/\alpha)^{2d+3}$. By a union bound, we have that for all $w \in C_w$, $v \in C_v$ and $T' \in C_{T'}$ that we have  (\ref{eq:thingie}) and (\ref{eq:thingie-2}) with $\eps/3$ in place of $\eps$ except with probability 
$O(1/\alpha)^{2d+3} \cdot \exp(-\Omega(N\eps^2/\ln(1/\eps))) = \exp(-\Omega(N\eps^2/\ln(1/\eps) - O((2d+3)\log(dN/\eps\tau))) \leq \tau/20$.

We have that for any unit vectors $v$, $w$ , there exist $v' \in C_v$ and $w' \in C_w$ such that $\|v-v'\|_2 \leq \alpha$ and $\|w-w'\|_2 \leq \alpha$.
For any $X' \in \R_{\beta',T'}$, we have that $v' \cdot X' - \alpha r \leq v \cdot X' \leq v' \cdot X' + \alpha r$ and the equivalent for $w$. We also have that $|(v' \cdot X')^2 - (v \cdot X')^2| \leq 2 |(v \cdot X')| \alpha r + \alpha r^2 \leq 3 \alpha r^2$.

We have that
\begin{align*}
\E_{R_{\beta',T'}}[ (v \cdot X)^2 (w \cdot X')^2 1_{|v \cdot X'| \leq T'}] 
& \leq \E_{R_{\beta',T'}}[ (v \cdot X)^2 (w \cdot X')^2 1_{|v' \cdot X'| \leq T' + \alpha r}] \\
& \leq  \E_{R_{\beta',T'}}[ ((v' \cdot X')^2 (w' \cdot X')^2  + O(\alpha r^4)) 1_{|v' \cdot X'| \leq T' + \alpha r}] \\
& \leq \E_{D_{\beta',T'}}[ ((v' \cdot X')^2 (w' \cdot X')^2 1_{|v' \cdot X'| \leq T' + \alpha r +\alpha}] + \eps/3 +O(\alpha r^4)) \\
& \leq \E_{D_{\beta',T'}}[ ((v \cdot X')^2 (w \cdot X')^2  + O(\alpha r^4)) 1_{|v \cdot X'| \leq T' + \alpha (2r+1)}] 
+ \eps/3 +O(\alpha r^4)) \\
& \leq \E_{D_{\beta',T'}}[ ((v \cdot X')^2 (w \cdot X')^2  1_{|v \cdot X'| \leq T' + \alpha (2r+1)}] 
+ \eps/3 +O(\alpha r^4))
\end{align*}
with a similar bound from below of the form $\E_{D_{\beta',T'}}[ ((v \cdot X')^2 (w \cdot X')^2  1_{|v \cdot X'| \leq T' - \alpha (2r+1)}] - \eps/3 -O(\alpha r^4))$.
Next we show that
\begin{align*}
\E_{D_{\beta',T'}}[ (v \cdot X')^2 (w \cdot X')^2  1_{T' \leq |v \cdot X'| \leq T' + \alpha (2r+1)}] 
& \leq O(T'^4) \cdot \Pr_{D_{\beta',T'}}[T' \leq |v \cdot X'| \leq T' + \alpha (2r+1)] \\
& \leq O(\alpha r^4\sqrt{\ln(1/\eps)}) \;.
\end{align*}
Thus, we have that (\ref{eq:thingie}) holds with the bound $\eps/3 + O(\alpha r^4\sqrt{\ln(1/\eps)}) \leq \eps$ for all unit vectors $w,v$ and all $2\sqrt{\ln(1/\eps)} \leq T' \leq 100\sqrt{\ln(1/\eps)}$ as required.
A similar proof shows that (\ref{eq:thingie-2}) also holds. 
\end{proof}

Next we show that these expectations are close under $D$ and $D_{\beta',T'}$:
\begin{lemma} \label{lem:cond-4th-moments}
For any  $2\sqrt{\ln(1/\eps)} \leq T' \leq 100\sqrt{\ln(1/\eps)}$,
let $D_{\beta',T'}$ be the true distribution over $(X,y)$ conditioned on $|y-\beta' \cdot X| \leq T' \sigma_{\beta'}$. Then for all $\beta' \in \R^n$ and all $2\sqrt{\ln(1/\eps)} \leq T' \leq 100\sqrt{\ln(1/\eps)}$ and all unit vectors $v$, we have that
$$|\E_{D}[((y -\beta' \cdot X) (v \cdot X))^2]
-\E_{D_{\beta',T'}}[((y -\beta' \cdot X) (v \cdot X))^2]| \leq O(\eps \log(1/\eps)^2 \sigma_{\beta'}^2)$$
and
$$|\E_{D}[((y -\beta' \cdot X) (v \cdot X))]
-\E_{D_{\beta',T'}}[((y -\beta' \cdot X) (v \cdot X))]| \leq O(\eps \log(1/\eps) \sigma_{\beta'}).$$
\end{lemma}
\begin{proof}
We show that the contribution of points with  $|y-\beta' \cdot X| > T' \sigma_{\beta'}$ to $\E_{D}[((y -\beta' \cdot X) (v \cdot X))^2]$ is small and use Cauchy-Schwarz.
Let $D'$ be $D$ conditioned on the negated condition that $|y-\beta' \cdot X| > T' \sigma_{\beta'}$ which happens with probability at most $\eps$. Let $\alpha = \Pr_D[|y-\beta' \cdot X| > T' \sigma_{\beta'}]$ and $t=\sqrt{2\ln(1/\alpha)} \sigma_{\beta'}$
and note that
\begin{align*}
 \alpha \E_{D'}[(y-\beta' \cdot X)^4] 
& = \alpha \int_{T=0}^\infty 4 T^3 \Pr_{D'}[|y-\beta' \cdot X| > T] dT \\
& \leq \int_{T=0}^\infty 4 T^3 \min\{\alpha, \Pr_{D}[|y-\beta' \cdot X| > T]\} \\
& \leq  \int_{T=0}^\infty 4 T^3 \min \{ 2 \exp(-T^2/2\sigma_{\beta'}^2),\alpha \} dT \\
& \leq \int_{T=0}^{t} 4 T^3 \alpha dT +  \int_{T=t}^{\infty} 4 T^3 \exp(-T^2/2\sigma_{\beta'}^2) dT \\
& \leq t^4 \alpha + \int_{T=t}^{\infty}  O(T \exp(-T^2/4\sigma_{\beta'}^2)) dT  \\
& \leq O(t^4 \alpha) = O(\alpha \log^2(1/\alpha) \sigma_{\beta'}^4) .
\end{align*}
By a similar proof,  $\alpha \E_{D'}[(v \cdot X)^4] \leq O(\alpha \log^2(1/\alpha))$. 
By Cauchy-Schwarz, $\alpha \E_{D}[((y -\beta' \cdot X) (v \cdot X))^2] \leq O(\alpha \log^2(1/\alpha) \sigma_{\beta'}^2)$ and applying it again  $\alpha |\E_{D}[((y -\beta' \cdot X) (v \cdot X))]| \leq O(\alpha \log(1/\alpha) \sigma_{\beta'})$.
Noting that $|\E_D[(y -\beta' \cdot X) (v \cdot X)]| \leq \sigma_{\beta'}$ and $\E_D[(y -\beta' \cdot X)^2 (v \cdot X)^2] \leq 3 \sigma_{\beta'}^2$, we have that
\begin{align*}
&(1- \alpha)|\E_{D}[((y -\beta' \cdot X) (v \cdot X))^2]
-\E_{D_{\beta',T'}}[((y -\beta' \cdot X) (v \cdot X))^2]|
\\
& \leq |\E_{D}[((y -\beta' \cdot X) (v \cdot X))^2]
-(1-\alpha) \E_{D_{\beta',T'}}[((y -\beta' \cdot X) (v \cdot X))^2]| + \alpha \E_{D}[((y -\beta' \cdot X) (v \cdot X))^2] \\
& \leq \alpha \E_{D'}[((y -\beta' \cdot X) (v \cdot X))^2] + 3 \alpha \sigma_{\beta}^2\\
& \leq O(\alpha \log^2(1/\alpha) \sigma_{\beta}^2) \\
& \leq O(\eps \log^2(1/\eps) \sigma_{\beta}^2)
\end{align*}
and
\begin{align*}
&(1- \alpha)|\E_{D}[((y -\beta' \cdot X) (v \cdot X))]
-\E_{D_{\beta',T'}}[((y -\beta' \cdot X) (v \cdot X))]|\\
& \leq |\E_{D}[((y -\beta' \cdot X) (v \cdot X))]
-(1-\alpha) \E_{D_{\beta',T'}}[((y -\beta' \cdot X) (v \cdot X))]| + \alpha |\E_{D}[((y -\beta' \cdot X) (v \cdot X))]|\\
& \leq \alpha |\E_{D'}[((y -\beta' \cdot X) (v \cdot X))]| +  \alpha \sigma_{\beta}\\
& \leq O(\alpha \log(1/\alpha) \sigma_{\beta}) \\
& \leq O(\eps \log(1/\eps) \sigma_{\beta}) .
\end{align*}
\end{proof}

\begin{corollary} \label{cor:hard-bit}
 Assuming Lemma \ref{lem:fourth-moments-with-d-samples} holds, 
 conditions 3 (iii) and (iv) of Definition \ref{def:good-set-2} hold.
 \end{corollary}
\begin{proof}
We need show show that these condtios hold for all $\beta',T'$.
First note that $\E_D[[(y-\beta' \cdot X)X] = (\beta-\beta')$. By Lemmas \ref{lem:cond-4th-moments},
\ref{lem:fourth-moments-with-d-samples} and the triangle inequality, we have 
$\left\|{\E_{R_{\beta',T'}}[(y-\beta' \cdot X)X] - (\beta-\beta')} \right\|_2\leq O(\eps \log(1/\eps) \sigma_{\beta'})$, which is (iii).

It follows using Claim \ref{claim:dec-chisq} that $\E_{D}[((y -\beta' \cdot X) X - (\beta-\beta')) ((y-\beta' \cdot X)X - (\beta-\beta'))^T]=\sigma_{\beta'}^2 I+ (\beta-\beta') (\beta-\beta')^T$. Again using Lemmas \ref{lem:cond-4th-moments},
\ref{lem:fourth-moments-with-d-samples} and the triangle inequality, we have  $\left\|M_{S,\beta'} - (\sigma_{\beta'}^2 I-{(\beta-\beta') (\beta-\beta')^T}) \right\|_2 \leq O( \log^2(1/\eps) \eps \sigma_{\beta}^2)$, where $M_{S,\beta'}=\E_{R_{\beta',T'}}[((y -\beta' \cdot X) X - (\beta-\beta')) ((y-\beta' \cdot X)X - (\beta-\beta'))^T]$ which is (iv).
\end{proof}

To make the robust variance estimation via interquartile range work, 
we need that
\begin{lemma} \label{lem:vc-assum}
With probability at least $1-\tau/10$, for all $v \in \R^d$ and $T \in \R$, 
$|\Pr_{(X, y) \sim S}[v\cdot X-y > T] - \Pr_{(X, y) \sim D}[v\cdot X-y ]| \leq \eps/10$.
\end{lemma}
\begin{proof}
Noting that the VC dimension of halfspaces over $(X,y)$ is $d+2$, 
the result follows from the VC inequality.
\end{proof}

\begin{proof}[Proof of Proposition~\ref{prop:good-set-nobeta}]

By a union bound, Claim \ref{claim:good-set-tail} and  Lemma \ref{lem:Gaussian-basics-a}, \ref{lem:y'-thing}, \ref{lem:fourth-moments-with-d-samples} and \ref{lem:vc-assum} all hold with probability at least $1-\tau$. We condition on the event that they do.
We have the following:
\begin{itemize}
\item 1 (i) and (ii) follow from Claim \ref{claim:good-set-tail} and (iii) and (iv) follow from Lemma \ref{lem:Gaussian-basics-a}.
\item 2 (i) and (ii) follow from Lemma  \ref{lem:y'-thing} and (iii) and (iv) from Lemma \ref{lem:Gaussian-basics-b}.
\item 3 (i) follows from Claim \ref{claim:good-set-tail} and the bound on $|y-\beta' \cdot X|$ in $R_{\beta',T'}$.
(ii) follows from Lemma \ref{lem:combined}, which required Claim \ref{claim:good-set-tail} and Lemma \ref{lem:y'-thing}.
(iii) and (iv) are given by Corollary \ref{cor:hard-bit} which requires Lemma \ref{lem:fourth-moments-with-d-samples}.

\item 4 is given by Lemma \ref{lem:vc-assum}.
\end{itemize}
This completes the proof of the proposition.
\end{proof}

\section{Proof of Proposition~\ref{prop:filter-lr-no-beta}}\label{sec:alg2-proof}
\subsection{Analysis if we Remove Samples at Any Step}
For this, we need to show that the conditions on the set $S$ 
given by Definition \ref{def:good-set-2} are sufficient to guarantee 
that the filter steps return a set $S''$ with $\Delta(S,S'') < \Delta(S,S')$.

Note that the conditions given in Definition \ref{def:good-set-2} 1 and 2, 
satisfy the requirements for the sub-gaussian filter, such as those given 
in~\cite{DiakonikolasKKL16-icml}, but with $\eps^2/T^2$ in place of $\eps/T^2$. 
This ensures that if either of the steps \ref{step:gaussian-filter-1} 
and \ref{step:gaussian-filter-2} return a subset $S''$, then $\Delta(S,S'') < \Delta(S,S')$.

For step \ref{step:lin-reg-filter}, we perform a filter similar to the one in the previous section. 
Note that if we replace $y$ by $y - \beta' \cdot X$, the filter is identical to the one in the previous section.
The samples in $S' \setminus U$ satisfy the conditions in Definition \ref{def:good-set-2} \ref{cond:good-for-lin-reg-filter} (i)-(iv). 
Conditions (i)-(iv) exactly correspond to (i) to (iv) of Definition \ref{def:good-set}.
Note that these conditions are sufficient for the proof of section \ref{ssec:alg} to apply and 
show that if this step removes samples from  $S' \setminus U$,  then  $\Delta(S \setminus U,S'' \setminus U) < \Delta(S \setminus U,S' \setminus U)$. Adding $U$ back, we obtain that $\Delta(S,S'') < \Delta(S,S')$.

\subsection{Analysis of Correctness if we Return $\beta'$}
We first show that after passing steps \ref{step:gaussian-filter-1} and \ref{step:gaussian-filter-2}, removing the samples in $U$ does not affect the expectation on $(y-\beta' \cdot X) X$ too much.

It is immediate from the definition of $U$ that when we pass step \ref{step:gaussian-filter-1}, we have $|U| \leq \eps |S'|$.

\begin{lemma} $|U|\E_U[(y-\beta' \cdot X)^2] \leq O(|S'|\eps \log(1/\eps) \sigma'^2)$ and $|U|\|\E_U[w \cdot X ]\|_2 \leq O(|S'|\eps \log(1/\eps))$ for $w$ as in Step 5.
\end{lemma}
\begin{proof}
If we pass the filter steps then $\E_{S'}[(y-\beta' \cdot X)^2] \leq (1+O(\eps \log(1/\eps))) \sigma'^2$ and
$\E_{S'}[(w \cdot X)^2] \leq 1+O(\eps \log(1/\eps))$.
However, we can lower bound the variance of $S'$ after removing a set of size $O(\eps|S'|)$ by Corollary A.11 of~\cite{DiakonikolasKKL16-icml} (which uses the same technique as Lemma \ref{lem:spectral-bound-L} in this paper), as $\E_{S' \setminus U}[(y-\beta' \cdot X)^2]] \geq (1-O(\eps \log(1/\eps))) \sigma'^2$ and $\E_{S' \setminus U}[(w \cdot X)^2] \leq 1-O(\eps \log(1/\eps))$. Taking the difference and scaling by $|S'|$ gives the lemma.
\end{proof}

\begin{corollary} \label{cor:pre-process}
$\| \E_{S' \setminus U}[(y-\beta' \cdot X) X]\|_2 \leq O(\eps \log(1/\eps) \sigma_{\beta'})$.
\end{corollary}
\begin{proof} This follows by Cauchy-Scwartz on the expectation over $U$. For the $w$ in step 5, which is the normalisation of  $\E_{S' \setminus U}[(y-\beta' \cdot X) X]$, we have
 $|U||\E_{U}[(y-\beta' \cdot X) (w \cdot X)]| \leq \sqrt{|U|\E_U[(y-\beta' \cdot X)^2]|U|\E_U[v \cdot X]} \leq O(|S'|\eps \log(1/\eps) \sigma')$.
Since $\E_{S'}[y-\beta' \cdot X]=0$, we have that $\E_{S' \setminus U}[(y-\beta' \cdot X) X]=(|U|/(|S'|-|U|))\E_{U}[(y-\beta' \cdot X) (w \cdot X)]$ and so $\| \E_{S' \setminus U}[(y-\beta' \cdot X) X]\|_2 \leq O((|S'|/(|S'|-|U|))\eps \log(1/\eps) \sigma')$.
 Since $|U| \leq \eps |S'|$ and $\sigma' \leq (1+O(\eps))\sigma_{\beta'}$, we are done.
\end{proof}

\begin{lemma} \label{lem:alg-gives}
 $\|\E_{S' \setminus U}[(y-\beta' \cdot X)X] - (\beta-\beta')\|_2 \leq  O(\eps \log(1/\eps) \sigma_{\beta'}).$
\end{lemma}
\begin{proof} The case of small spectral norm in section \ref{ssec:alg} gave that $\|\beta_{S'}-\beta\|_2/\sigma_y \leq O(\eps \log(1/\eps))$. Translating that into the notation of this section and applying it to $S' \setminus U$ with $y'=y - \beta' \cdot X$ in place of $y$ gives that $\|\E_{S' \setminus U}[(y-\beta' \cdot X) X] - (\beta-\beta')\|_2 \leq  O(\eps \log(1/\eps) \sigma_{\beta'})$. \end{proof}

Combining Corollary \ref{cor:pre-process} and Lemma \ref{lem:alg-gives}, we obtain
\begin{corollary}
$\|\beta-\beta'\|_2 \leq O(\eps \log(1/\eps) \sigma)$.
\end{corollary}
\begin{proof}
By the triangle inequality, $\|\beta-\beta'\|_2 \leq O(\eps \log(1/\eps) \sigma_{\beta'})$. However, recall that
$\sigma_{\beta'} = \sqrt{\sigma^2 + \|\beta-\beta'\|_2} \leq \sigma +  \|\beta-\beta'\|_2$. We thus obtain that
$\|\beta-\beta'\|_2 \leq  O(\eps \log(1/\eps) \sigma) /(1-O(\eps \log(1/\eps))) = O(\eps \log(1/\eps) \sigma)$ as long as $\eps$ is sufficiently small.
\end{proof}

\section{Proof of Theorem~\ref{thm:SQ-lb}: Statistical Query Lower Bounds} \label{sec:sq-app}
We restate Theorem~\ref{thm:SQ-lb} below for convenience:

\vspace{.3cm}\noindent \textbf{Theorem~\ref{thm:SQ-lb}} \emph{
No algorithm given statistical query access to $Q'$, defined as above with unknown noise 
and unknown variances $(1/2) I \preceq \Sigma \preceq I$ and $\sigma^2 \leq 1$, 
gives an output $\beta'$ with $\|\beta'-\beta\|_2 \leq o(\sqrt{\eps})$ on all instances
unless it uses more than  $2^{\Omega(d^c)} d^{4c-2}$ calls to the 
$$\mathrm{STAT}\left(O(d)^{2c-1} e^{O(1/\eps)} \right) \textrm{ or } \mathrm{VSTAT}\left( O(d)^{2-4c}/e^{O(1/\eps)} \right)$$ 
oracles for any $c > 0$.}

Recall that we use the construction of \cite{DKS17-sq}, which intuitively says that if we have a distribution which is standard Gaussian in all except one direction then if the low degree moments match the standard Gaussian, then that direction is hard to find with an SQ algorithm. The idea is that, if we consider consider $X$ conditioned on $y$ for non-zero $\beta$, then $X$ has a non-zero mean in the $\beta$ direction. By adding noise, we can make $X$ conditioned on any $y$ agree with the first three moments of $\normal(0,I)$ and like the construction of \cite{DKS17-sq}, be a standard Gaussian in all except one direction. Then we can show that we cannot find the direction of $\beta$ with an SQ algorithm. 

	The further the mean of $X$ conditioned on $y$ is from $0$, the more noise needs to be added to match the first three moments. Lemma~\ref{lem:A-properties}, which is the main lemma of the lowerbound proof, shows that we can match the first three moments by adding $O(\mu^2)$ fraction of noise when the $X|y$ has mean $\mu$ in the $\beta$ idrection. As long as $\|\beta\|_2=O(\sqrt{\eps})$, after taking the integral over $y$, the overall noise added will still be smaller than $\eps$. Lemma~\ref{lem:uvchi2} establishes the upperbound of the statistical correlation between a pair of distributions under our construction, which allows the classical statical query scheme to be applied to yield the lowerbound.  

The rest of the section formally proves Theorem~\ref{thm:SQ-lb}. To start, recall that $X|y$ is distributed as Gaussian, restated as below: 

\vspace{.3cm}\noindent \textbf{Lemma~\ref{lem:ycondx}} \emph{ Let $Q$ be the joint distribution of $(X,y)$ where $X \sim \normal(0,\Sigma)$ and $y|X \sim \beta^T X + \eta$ where $\beta$ is unknown an $\eta \sim \normal(0,\sigma^2)$. Then $y \sim \normal(0, \sigma_y^2)$ where $\sigma_y^2 = \beta^T \Sigma \beta+\sigma^2$ and $X|y \sim \normal(\frac{\Sigma\beta y}{\sigma_y}, \Sigma-\frac{(\Sigma\beta)(\Sigma\beta)^T}{\sigma_y})$. }

For the simplicity of our construction, we let the variance of $X$ to be $1$ in all except $\beta$ direction while the $\beta$ direction will have smaller variacne. This is because if the corruption affects the mean in $\beta$ direction by much, they also increase the variance significantly. However to match the second moment, we need to keep the corrupted variance as $1$. For the ease of computation, we will take $y$ to have variance $1$ and $X|y$ to have covariance $I-(1/3)vv^T$.

\begin{lemma} If we set $\beta = c_1\sqrt{\eps}v$, $X \sim \normal(0,\Sigma)$ where $\Sigma = I-c_2 vv^T$ and $\sigma^2=1-\beta\Sigma\beta$ for constants $c_1,c_2 \geq 0$, then for any $0 < c_1 \leq 1/10$, there exists a $c_2>0$, such that $\sigma_y=1$ and $X|y \sim \normal(c_1(1-c_2)\sqrt{\eps}yv,I-(1/3)vv^T)$ \end{lemma}
\begin{proof}
By the previous lemma, we have $X|y\sim \normal(c_1(1-c_2)\sqrt{\eps}yv, I - (c_2+(c_1(1-c_2))^2\eps)vv^T)$. Given arbitrary $c_1<1/10$, there exists a $c_2$ such that $(c_2+(c_1(1-c_2))^2\eps)=1/3$.
\end{proof}
We take the joint distribution of $(X,y)$ given by the above lemma to be $Q_v$. We now need to define the corrupted distribution $Q'_v$. We want $X|y$ for any $y$, in $Q'_v$ to be a distribution of the form of our SQ lower bound construction, for which we need a one dimensional distribution that agrees with the first three moments of $\normal(0,1)$, that is close to the distribution of $(v\cdot X)|y$ under $Q_v$, which is $\normal(c_1(1-c_2)\sqrt{\eps}y, 2/3)$.

\begin{lemma} \label{lem:A-properties}
For any $\eps > 0$, $\mu \in \R$, there is a distribution $A_{\mu}$ such that $A_{\mu}$ agree 
with the first $3$ moments of $\normal(0,1)$ and $A_{\mu} =(1-\eps_\mu) \normal(\mu,2/3)+\eps_\mu B_\mu$ 
for some distribution $B_{\mu}$ and $\eps_{\mu}$ satisfying:
\begin{itemize} 
\item If $|\mu| \geq \sqrt{\eps}/10000$, then $\eps_{\mu}/(1-\eps_{\mu}) \leq 36 \mu^2$ 
and $\chi^2(A_{\mu},\normal(0,1)) = e^{O(\max(1/\mu^2,\mu^2))}$.
\item If $|\mu| < \sqrt{\eps}/10000$, then $\eps_{\mu}=\eps$ and  $\chi^2(A_{\mu},\normal(0,1))=e^{O(1/\eps)}$.
\end{itemize}
\end{lemma}

Subsection \ref{app-sec:SQ-moment-matching} will be devoted to the proof of the above lemma. Here we show that it suffices to prove Theorem \ref{thm:SQ-lb}. Similarly to the construction in \cite{DKS17-sq}, we define $P_{\mu,v}(x)=A_\mu(v.x) \exp(-||x-(v.x) x||_2^2/2)/\sqrt{2 \pi}$. By Lemma 3.4 of that paper, since $A_\mu$ agrees with the first $3$ moments, we have that for unit vectors $v,v'$,
\begin{equation} \label{eq:old-construction-guarantee}
|\chi_{\normal(0,I)}(P_{\mu, v},P_{\mu, v'})| =  O(\cos^4 \theta \chi^2(A_{\mu},\normal(0,1))) \; .
\end{equation}
Now we need to define a $Q'_v(X,y)$ such that $X|y \sim P_{\mu,v}(X)$ that is a contaminated version of $Q_v$:
\begin{lemma} \label{lem:noisy-def}
If we define $Q'_v(X,y)=P_{\mu(y),v}(X) R(y)$, 
where $$R(y)=G(y)/\left((1-\eps_{\mu(y)}) \int_{-\infty}^\infty G(y')\\/(1-\eps_{\mu(y')}) dy' \right)$$ and $\mu(y)=c_1(1-c_2)\sqrt{\eps}y$, then
$Q'_v(X,y)$ is a distribution with $Q' = (1-\eps)Q+\eps N$ for some distribution $N$ and under $Q'_v$, $X|y \sim P_{\mu(y),v}(X)$. \end{lemma}
\begin{proof}
First, to show that $R(y)$ and so $Q'_v(x,y)$ are well defined, we need to show that $\int_{-\infty}^\infty G(y')/(1-\eps_{\mu(y')}) dy'$ is finite. Indeed we have that
\begin{align*} 
\int_{-\infty}^\infty G(y')/(1-\eps_{\mu(y')}) dy' & = \int_{-\infty}^\infty G(y')(1+\frac{\eps_{\mu(y)}}{1-\eps_{\mu(y')}}) dy' \\
& \leq 1 + 36(c_1(1-c_2))^2 \eps \int_{-\infty}^\infty G(y') y'^2 dy' +\eps/10000 \\
& \leq  1+ (36(c_1(1-c_2))^2 + 1/1000) \eps \\
& \leq 1/(1-\eps) \;,
\end{align*}
where we have applied Lemma~\ref{lem:A-properties} and the fact that $c_2 > 0, c_1 \leq 1/40$. 
We have that $R(y)$ is non-negative and integrates to $1$. Thus $Q'$ is the joint distribution of $X$ and $y$, $y$ has the distribution $R$ with pdf $R(y)$ and $X|y \sim P_{\mu(y),v}$

 Since for any $\mu, X$, $A_{\mu}(X) \geq (1-\eps_{\mu}) N_{\mu, 2/3}(X)$, where we use $N_{\mu, \Sigma}(X)$ to denote the pdf function of distribution $\normal(\mu,\Sigma)$, we have that $P_{\mu,v}(X) \geq (1-\eps_\mu) N_{\mu v,I-(1/3)v v^T}(X)$.
Since $R(y) \geq (1-\eps)G(y)/(1-\eps_\mu(y))$, we have furthermore that 
$Q'_v(X,y) = P_{\mu(y),v}(X) R(y) \geq (1-\eps) N_{\mu(y) v,I-(1/3)vv^T}(X)G(y)=(1-\eps)Q(x,y)$. 
We can thus write $Q' = (1-\eps)Q+\eps N$ for the distribution $N$ with pdf 
$\normal(x,y) = (1/\eps)(Q'(x,y) -(1-\eps) Q(x,y))$.
\end{proof}

\begin{lemma}\label{lem:uvchi2} For $Q_v'$ as in Lemma \ref{lem:noisy-def}, we have
$$
\chi_S(Q'_{v}, Q'_{v'}) = e^{O(1/\eps)} (v^T v')^4
$$
where $S$ is the joint distribution of $x$ and $y$ when they are independent and $x \sim \normal(0,I)$ and $y \sim R$. 
\end{lemma}
\begin{proof}
The chi-square divergence is expressed as: 
\begin{align*}
& 1 + \chi_S(Q'_{v}, Q'_{v'})\\
= & \int Q'_v(x,y) Q'_{v'}(x,y)/S(x,y) dx dy\\
= & \int P_{\mu,v}(x) P_{\mu,v'}(x) R^2(y)/{\bf G}(x) R(y) dx dy  \\
= & \int (1+ \chi_{\normal(0,I)}(P_{\mu,v}(X), P_{\mu,v'}(X))) R(y) dy\\
= & 1+ \int_{-\infty}^{\infty} O((v^T v')^4 \chi^2(A_{\mu(y),\eps},\normal(0,1))) R(y) dy \\
= &1 + O((v^T v')^4 \int_{-\infty}^{\infty} \chi^2(A_{\mu(y),\eps},\normal(0,1))) R(y) dy),
\end{align*}
where we have applied Lemma 3.4 of~\cite{DKS17-sq}. Recall that $\mu(y)=c_1(1-c_2)\sqrt{\eps}y$ 
and by Lemma \ref{lem:A-properties}, If $|\mu| \geq \sqrt{\eps}/10000$, 
then  $\chi^2(A_{\mu}, \normal(0,1)) = e^{O(\max(1/\mu^2,\mu^2)} \leq e^{O(\max(1/\eps,c_1(1-c_2) \eps y^2)} $ 
and if $|\mu| < \sqrt{\eps}/10000$, then  $\chi^2(A_{\mu}, \normal(0,1))=e^{O(1/\eps)}$. 
We thus have that
\begin{align*} 
\int_{-\infty}^{\infty} \chi^2(A_{\mu(y),\eps},\normal(0,1)) R(y) dy& \leq e^{O(1/\eps)} + O\left( \int_{c_1(1-c_2)/10000}^{\infty} y^2 R(y) dy \right) \\
& \leq e^{O(1/\eps)} + O\left( \int_{c_1(1-c_2)/10000}^{\infty} y^2 G(y) (1-\eps)/(1-\eps_{\mu(y)}) dy \right) \\
& \leq e^{O(1/\eps)} + O\left( \int_{c_1(1-c_2)/10000}^{\infty} y^2 G(y) (1-\eps)36 \mu(y)^2 dy \right) \\
& \leq e^{O(1/\eps)} + O\left( \int_{c_1(1-c_2)/10000}^{\infty} y^4 G(y) (1-\eps)36 c_1^2(1-c_2)^2 \eps dy \right) \\
& \leq e^{O(1/\eps)} + O(1) \leq e^{O(1/\eps)} \;.
\end{align*}
We thus have that $\chi_S(Q'_{v}, Q'_{v'}) \leq (v^T v')^4 e^{O(1/\eps)}$, 
as required.
\end{proof}

\begin{proof}[Proof of Theorem \ref{thm:SQ-lb} given Lemma \ref{lem:A-properties}]
The proof now follows that of Proposition 3.3 of \cite{DKS17-sq}. By Lemma 3.7 of~\cite{DKS17-sq}, for any $0 < c < 1/2$, there is a set $S$ of at least $2^{d^c}$ unit vectors in $\mathbb{R}^d$ such that for each pair of distinct $v, v'\in S$, it holds $|v \cdot v'| \le  O(d^{c-1/2})$.  Then, by Lemma~\ref{lem:uvchi2}, we have that for $v,v' \in S$ with $v\ne v'$,it holds $\chi^2_{\normal(0,I)}(P_v, P_{v'} ) \le {(v^T v')}^4e^{O(1/\eps)} = d^{4c-2}e^{O(1/\eps)}$ . And for $v= v' \in S$, it holds that $\chi^2_{\normal(0,I)}(P_v, P_{v'} ) \le e^{O(1/\eps)}$.

We thus have that the set of $P_v$ for $v \in S$ is $(\gamma,\beta)$-corrlated for $\gamma=d^{4c-2}e^{O(1/\eps)}$, $\beta=e^{O(1/\eps)}$. Thus applying Lemma 2.12 of~\cite{DKS17-sq}, 
we obtain that any SQ algorithm requires at least $2^{\Omega(d^c)} d^{4c-2}$ calls to the 
$$\mathrm{STAT}\left(O(d)^{2c-1} e^{O(1/\eps)} \right) \textrm{ or } \mathrm{VSTAT}\left( O(d)^{2-4c}/e^{O(1/\eps)} \right)$$ 
oracle to find $v$ and therfore $\beta$ within better than $O(\sqrt{\eps})$.
Note that $2^{\Omega(n^{c/2})} \geq \Omega(d^4)$ for any $c$. Hence, the total number of required queries is at least 
$2^{\Omega(d^{c/2})}$.
This completes the proof.

\end{proof}

\section{Proof of Lemma \ref{lem:A-properties}} \label{app-sec:SQ-moment-matching}
We restate Lemma~\ref{lem:A-properties} here for convenience. 

\vspace{.5cm}\noindent \textbf{Lemma~\ref{lem:A-properties}.} \emph{For any $\eps > 0$, $\mu \in \R$, there is a distribution $A_{\eps,\mu}$(also written as $A_{\mu}$ for simplicity) such that $A_{\mu}$ agree 
with the first $3$ moments of $\normal(0,1)$ and $A_{\mu} =(1-\eps_\mu) \normal(\mu,2/3)+\eps_\mu B_\mu$ 
for some distribution $B_{\mu}$ and $\eps_{\mu}$ satisfying:
\begin{itemize} 
\item If $|\mu| \geq \sqrt{\eps}/10000$, then $\eps_{\mu}/(1-\eps_{\mu}) \leq 36 \mu^2$ 
and $\chi^2(A_{\mu},\normal(0,1)) = e^{O(\max(1/\mu^2,\mu^2))}$.
\item If $|\mu| < \sqrt{\eps}/10000$, then $\eps_{\mu}=\eps$ and  $\chi^2(A_{\mu},\normal(0,1))=e^{O(1/\eps)}$.
\end{itemize}
}

We split this into a number of cases, each of which will be a mixture of three Gaussians. The parameters and weights of these Gaussians will need to be chosen to make the first three moments the same as that of $\normal(0,1)$ which requires satisfying a cubic equation.

We first deal with the case when $\mu \geq \sqrt{\eps}/10000$. We will need to further split this into cases.

\begin{lemma}\label{lem:peps}
We have the following:
\begin{itemize}
\item
For $0 \le \eps\le 0.42$, the distribution
$P_{1,\eps} = \frac{1}{9}\eps \normal(-\frac{1}{\sqrt{\eps}},a)+\frac{8}{9}\eps \normal(\frac{1}{2\sqrt{\eps}},b)+(1-\eps)\normal(-\frac{\sqrt{\eps}}{3(1-\eps)},c)$, where $a,b,c$ are defined as
\begin{align*}
a= \frac{2}{3}-\frac{1}{3(1-\eps)}-\frac{2\eps}{27(1-\eps)^2}\\
b=\frac{2}{3}-\frac{1}{12(1-\eps)}+\frac{\eps}{108(1-\eps)^2}\\
c=\frac{2}{3},
\end{align*}
has first moment $0$, second moment $1$, third moment $0$ with $a,b\in (0,2)$.

\item For $0.35\le \eps \le 0.78$, the distribution
$P_{2,\eps} = \frac{1}{9}\eps \normal(-\frac{2}{3\sqrt{\eps}},a)+\frac{8}{9}\eps \normal(\frac{1}{3\sqrt{\eps}},b)+(1-\eps)\normal(-\frac{2\sqrt{\eps}}{9(1-\eps)},c)$ where $a,b,c$ are defined as
\begin{align*}
a= -\frac{-162 \epsilon ^3+161 \epsilon ^2+144 \epsilon -135}{243 (\epsilon -1)^2 \epsilon }\\
b=-\frac{-648 \epsilon ^3+1121 \epsilon ^2-342 \epsilon -135}{972 (\epsilon -1)^2 \epsilon }\\
c=\frac{2}{3},
\end{align*}
has first moment $0$, second moment $1$, third moment $0$ with $a,b\in (0,2)$. 

\item
For $0.49 \leq \eps \leq 1$, the distribution 
$P_{3,\eps} = \frac{1}{8}(1-\eps) \normal(-\frac{1}{\sqrt{\frac{9}{8}(1-\eps)}},a)+(1-\eps) \normal(\frac{\sqrt{2}}{3\sqrt{1-\eps}},b)+(\frac{9}{8}\eps-\frac{1}{8})\normal(-\frac{\sqrt{8(1-\eps)}}{(9\eps-1)},c)$ , where $a,b,c$ are defined as
\begin{align*}
a=\frac{2 \left(27 \eps^2-18 \eps+7\right)}{3 \left(27 \eps^2-12
   \eps+1\right)}\\
b=\frac{2}{3}\\
c=\frac{2 \left(243 \eps^3-105 \eps^2-15 \eps+5\right)}{3 (3
   \eps-1) (9 \eps-1)^2}
\end{align*}
has first moment $0$, second moment $1$, third moment $0$ with $a,b\in (0,2)$.

\end{itemize}
\end{lemma}

\begin{proof}
We verified these facts with the symbolic computation function of Mathematica.
\end{proof}

 Then we consider the remaining case when $\mu<\sqrt{\eps}/10000$.
\begin{lemma}\label{lem:smally}
Given $\mu<\sqrt{\eps}/10000$, the distribution $P_{4,\mu,\eps} = \eps_1 \normal(\mu_1,\sigma_1)+(1-\eps_1) \normal(\mu_2,\sigma_2)+(1-\eps) \normal(\mu,2/3)$ has first three moments as $0,1,0$ and that  $|\mu_1|,|\mu_2|<2/\sqrt{\eps}$ $0.9<\sigma_1,\sigma_2<1.1$.
\end{lemma}
\begin{proof}
Let $\mu_2 =  -c \mu_1$, to simplify the problem, we require $\eps_1 \mu_1^3+\eps_2 \mu_2^3=0$ and hence $\eps_1 = c^3 \eps_2=\frac{c^3}{1+c^3}\eps$. The first moment equation requires 
$$\frac{c^3}{1+c^3}\eps \mu_1- \frac{1}{1+c^3}\eps c\mu_1+(1-\eps)\mu=0 \;.$$
The second moment condition requires
$$
\frac{c^3}{1+c^3}\eps(\mu_1^2+\sigma_1)+ \frac{1}{1+c^3}\eps(c^2\mu_1^2+\sigma_2)+(1-\eps)(\mu^2+2/3)=1.
$$
The third moment condition requires
\begin{align*}
\frac{c^3}{1+c^3}\eps(\mu_1^3 + 3\mu_1 \sigma_1) - \frac{1}{1+c^3}\eps(c^3\mu_1^3 + 3c\mu_1 \sigma_2)+(1-\eps)(\mu^3+2\mu).\\
=\frac{3c^3}{1+c^3}\eps\mu_1 \sigma_1 - \frac{3c}{1+c^3}\eps\mu_1 \sigma_2+(1-\eps)(\mu^3+2\mu)=0.
\end{align*}
To simplify the problem, we require 
$$
1 - \frac{c^3+c^2}{1+c^3}\eps \mu_1^2-(1-\eps)(\mu^2+2/3)=\eps \;,
$$
which also implies
$$
\frac{c^3}{1+c^3}\sigma_1+ \frac{1}{1+c^3}\sigma_2=1
$$
by the second moment equation. We can solve for the value of $c$ and $\mu_1$ together using the first moment condition. The solution is the following:
\begin{align*}
\mu_1 = \frac{\left(1- 3 \mu ^2\right) \left(3 \mu ^2 (1-\epsilon)-\sqrt{3} \sqrt{\mu ^2
   (\epsilon -1) \left(9 \mu ^2+3 \mu ^2 \epsilon -4 \epsilon \right)}\right)}{6 \mu 
   \left(\epsilon -3 \mu ^2\right)}\\
c = \frac{\left(2-3 \mu ^2\right) \epsilon -3 \mu ^2-\sqrt{3} \sqrt{\mu ^2 (\epsilon -1)
   \left(9 \mu ^2+3 \mu ^2 \epsilon -4 \epsilon \right)}}{2 \left(\epsilon -3 \mu
   ^2\right)}
\end{align*}
The range of $\sqrt{-
   \left(9 \mu ^2+3 \mu ^2 \epsilon -4 \epsilon \right)}$ is $[1.99\sqrt{\eps},2\sqrt{\eps}]$. Assume that $\mu\le \sqrt{\eps}/10000$, we have that $0.99<c<1.01$ and $-\sqrt{\frac{4(1-\eps)}{\eps}}<\mu_1<-\sqrt{\frac{(1-\eps)}{2\eps}}$. Plugging the range of $\mu_1$ into the second moment condition yields:
$$
   \frac{3c^3}{1+c^3} \sigma_1 - \frac{3c}{1+c^3} \sigma_2=\frac{(1-\eps)(\mu^3+2\mu)}{\eps\mu_1}\le 0.01. 
$$
Solving the linear equations we get $0.9<\sigma_1,\sigma_2<1.1$.
\end{proof}

Finally we can combine the corrupted distributions constructed in different cases into a single corrupted distribution $A_{\mu,\eps}$ by the following definition.
\begin{definition}
\[
A_{\mu,\eps} = 
\begin{cases}
P_{1,\eps_\mu} \text{ where } \eps \text{ is the solution of } \frac{\sqrt{\eps_\mu}}{3(1-\eps_\mu)}=\mu& \sqrt{\eps}/10000 \le \mu\le 0.3\\
P_{2,\eps_\mu} \text{ where } \eps \text{ is the solution of } \frac{\sqrt{\eps}}{6(1-\eps_\mu)}=\mu& 0.3<\mu<0.7\\
P_{3,\eps_\mu} \text{ where } \eps_\mu=1-\frac{2}{9\mu^2} & 0.7\le \mu \\
P_{4,\mu, \eps}  &\mu \le \sqrt{\eps}/10000\\
\end{cases}
\]
\end{definition}
\begin{lemma}
$A_{\mu,\eps}$ is well-defined for any $\mu$ and has first three moments as $0,1,0$.
\end{lemma}
\begin{proof}
We need to verify that in each case $\eps_\mu$ lies in the correct range for $P_{1,\eps_\mu},P_{2,\eps_\mu},P_{3,\eps_\mu}$ to be well-defined. When $\mu\le 0.3$, the solution of $\frac{\sqrt{\eps}}{3(1-\eps)}=\mu$ is less than $0.35$. When $0.3 \le \mu\le 0.7$, the solution of $\frac{2\sqrt{\eps}}{9(1-\eps)}=\mu$ is between $0.48$ and $0.75$. When $0.7 \le \mu$, $1-\frac{2}{9\mu^2}$ is greater than $0.54$.
\end{proof}

\begin{lemma}\label{lem:amu}
$A_{\eps,\mu}=(1-\eps_\mu)\normal(\mu,2/3)+\eps_\mu B_{\eps,\mu}$ for some distribution $B_{\eps,\mu}$ and when $\mu \geq \sqrt{\eps}/10000$, $36\mu^2\ge \frac{{\eps_\mu}}{(1-\eps_\mu)}$.
\end{lemma}
\begin{proof}
When $\mu \le 0.3$, we have $\frac{\sqrt{\eps}}{3(1-\eps)}=\mu$ which yields $\frac{{\eps}}{(1-\eps)^2}=9\mu^2\ge \frac{{\eps}}{(1-\eps)}$. When $0.3 \le \mu \le 0.7$, we have $\frac{2\sqrt{\eps}}{9(1-\eps)}=\mu$ which yields $\frac{{\eps}}{(1-\eps)^2}=\frac{81}{4}\mu^2\ge \frac{{\eps}}{(1-\eps)}$. When $\mu \ge 0.7$, we have $\frac{\eps}{(1-\eps)} = \frac{9}{2}\mu^2-1$.
\end{proof}

Finally, to complete the proof of Lemma \ref{lem:A-properties}, we will need to bound $\chi^2(A_{\eps,\mu}, \normal(0,1))$. Notice that by Fact~\ref{fact:chi2} and Fact~\ref{fact:chi2cor}, we have $\chi^2(\normal(\mu_1,a),\normal(\mu_2,b)=e^{O(\mu_1^2+\mu_2^2)}$ and $\chi^2_{\normal(0,1)}(\normal(\mu_1,a),\normal(\mu_2,b)=e^{O(\mu_1^2+\mu_2^2)}$ for constant $a,b\in (0,2)$. In the case where $\mu\le \sqrt{\eps}/10000$, it is straightforward to verify that all the means of the Gaussian distributions are $O(\max(\frac{1}{\mu^2},\mu^2))$. Hence by applying Fact~\ref{fact:mixchi2} repeatedly, we claim that $\chi^2(A_{\eps,\mu},\normal(0,1))=e^{O(\max(\frac{1}{\mu^2},\mu^2))}$ when $\mu\ge \sqrt{\eps}/10000$. In the case where  $\mu\le \sqrt{\eps}/10000$, the means of all the gaussian of $P_{4,\mu,\eps}$  are all bounded by $O(1/\sqrt{\eps})$ regardless of $\mu$ and hence we have $\chi^2(P_{4,\mu,\eps},\normal(0,1))=e^{O(1/\eps)}$. Hence, the proof of Lemma~\ref{lem:A-properties} is complete.

%\begin{fact}
%The first moment of $\normal(\mu,\sigma^2)$ is $\mu$, the second moment is $\mu^2+\sigma^2$, the third moment is $\mu^3+3\mu\sigma^2$.
%\end{fact}

The following three technical facts regarding the chi-square distance and the Gaussian distribution can be verified easily, we omit some of the proof.
\begin{fact}\label{fact:mixchi2}
For distributions $B,C,D$ and $w\in [0,1]$, we have that $\chi^2 (wB+(1-w)C,D)=w^2\chi^2(B,D)+ (1 - w)^2\chi^2(C,D) + 2w(1 - w)\chi_D(B,C)$.
\end{fact}
\begin{proof}
\begin{align*}
1 + \chi^2(wB + (1 - w)C,D) =   \int (wB(x) + (1 - w)C(x))^2/D(x)dx\\
=w^2 \int B^2(x)/D(x)dx+(1-w)^2  \int C(x)^2/D(x)dx +2w(1-w)  \int B(x)C(x)/D(x)dx \\
= w^2(1 + \chi^2(B, D)) + (1 - w)^2(1 + \chi^2(C, D)) + 2w(1 - w)(1 + \chi_D(B, C))\\
= 1 + w^2\chi^2(B,D) + (1 - w)^2\chi^2(C,D) + 2w(1 - w)\chi_D(B,C) .
\end{align*}
\end{proof}
\begin{fact}\label{fact:chi2}
$$
\chi^2(\normal(\mu_1,\sigma_1^2),\normal(\mu_2,\sigma_2^2)) =  \frac{\sigma_2^2}{\sigma_1\sqrt{2\sigma_2^2 - \sigma_1^2}}
    e^{\frac{(\mu_1-\mu_2)^2}{2\sigma_2^2 - \sigma_1^2}} - 1.
$$
\begin{fact}\label{fact:chi2cor}
$$
\chi^2_{\normal(0,1)}(\normal(\mu_1, \sigma_1^2),\normal(\mu_2, \sigma_2^2))=\frac{\exp \left(-\frac{\mu_1^2 \left(\sigma_2^2-1\right)+2 \mu_1 \mu_2+\mu_2^2 \left(\sigma_1^2-1\right)}{2 \sigma_1^2 \left(\sigma_2^2-1\right)-2 \sigma_2^2}\right)}{ \sqrt{\sigma_1^2+\sigma_2^2-\sigma_1^2\sigma_2^2}}-1
$$
\end{fact}

\end{fact}

%\begin{lemma}\label{lem:chi}
%$\chi^2(B_{\mu,\eps},\normal(0,1))=e^{O(1/\eps)}$. $\chi^2(A_{\mu,\eps},\normal(0,1)) = e^{O(max(1/\mu^2,\mu^2)}$.
%\end{lemma}
\iffalse
\newpage

\begin{proof}
By Lemma~\ref{lem:chi}, the integral $\int_0^{1/10000} \chi^2(B_{\mu(y),\eps},\normal(0,1)) G'(y) dy$ is $e^{O(1/\eps)}$. And the integral $\int_{1/10000}^\infty O(\chi^2(A_{\mu(y)},\normal(0,1))) G'(y) dy$ is 
$$
O(\int e^{C(c_1(1-c_2))^2\eps y^2-y^2/2} y^2 dy)
$$
where C is an absolute constant. We can pick choose a small constant $c_1$ such that $C(c_1(1-c_2))^2<1/4$ and hence we have the above formula is smaller than $e^{O(1/\eps)}$
\end{proof}
\fi

%\input{filterGaussianAppendix}

\end{document}